\documentclass[11pt]{article}

\usepackage{fullpage}
\usepackage[round]{natbib}

\usepackage{amsmath,amsthm,amsfonts,amssymb}
\usepackage{amsmath}
\usepackage{hyperref}
\usepackage{color}
\usepackage{mathrsfs}
\usepackage{bm}
\usepackage{multirow}
\usepackage{booktabs}
\usepackage{makecell}
\usepackage{graphicx}
\usepackage{subfigure}
\usepackage{caption}
\usepackage{thmtools}
\usepackage{thm-restate}
\usepackage{setspace}
\usepackage{makecell}
\definecolor{light-gray}{gray}{0.85}
\usepackage{colortbl}
\usepackage{xcolor}
\usepackage{hhline}

\usepackage[round]{natbib}
\usepackage{hyperref}       
\usepackage{url}    
\usepackage{booktabs}       
\usepackage{nicefrac}       
\usepackage{microtype}  

\usepackage{amsmath,amsthm,amsxtra,graphicx,verbatim,epsfig,color,enumerate,array,mathtools,dsfont,multicol,mathrsfs,subfigure}

\usepackage[ruled,linesnumbered,vlined]{algorithm2e}
\usepackage{algpseudocode}
\SetKwInOut{Parameter}{parameter}

\hypersetup{
    colorlinks,
    linkcolor={blue!50!black},
    citecolor={blue!50!black},
}
\colorlet{linkequation}{blue}

  {\list{}{\leftmargin=0.3in\rightmargin=0.3in}\item[]}%
  {\endlist}
  
\usepackage[ruled,linesnumbered,vlined]{algorithm2e}
\usepackage{algpseudocode}
\SetKwInOut{Parameter}{parameter}
\usepackage{amsfonts}

\newtheorem{theorem}{Theorem}[section]
\newtheorem{lemma}[theorem]{Lemma}

\newtheorem{remark}[theorem]{Remark}

\theoremstyle{definition}
\newtheorem{definition}[theorem]{Definition}

\newtheorem{assumption}{Assumption}

\newcommand{\lep}[1]{\mathop  \le \limits^{(#1)}}
\newcommand{\gep}[1]{\mathop  \ge \limits^{(#1)}}
\newcommand{\ep}[1]{\mathop  = \limits^{(#1)}}
\newcommand{\ex}[1]{\mathbb{E}\left[ #1 \right] }
\newcommand{\bp}{{\psi} }
\newcommand{\vp}{{\varphi} }
\newcommand{\bt}{{\theta} }

\newcommand{\priv}{\rmfamily\scshape Privatizer}
\newcommand{\reg}{\rmfamily\scshape Regularizer}

\newcommand{\cA}{\mathcal{A}}

\newcommand{\cE}{\mathcal{E}}

\newcommand{\cM}{\mathcal{M}}
\newcommand{\cN}{\mathcal{N}}

\newcommand{\cU}{\mathcal{U}}

\newcommand{\cX}{\mathcal{X}}

\newcommand{\Real}{\mathbb{R}}

\newcommand{\norm}[1]{\left\lVert#1\right\rVert}

\newcommand{\prob}[1]{\mathbb{P}\left[{#1}\right]}

\newcommand{\inner}[2]{\langle #1, #2 \rangle}

\newcommand{\argmax}{\mathop{\mathrm{argmax}}}

\newcommand{\argmin}{\mathop{\mathrm{argmin}}}

\newcommand{\beq}{\begin{equation}}
\newcommand{\eeq}{\end{equation}}
\newcommand{\beqn}{\begin{equation*}}
\newcommand{\eeqn}{\end{equation*}}
\newcommand{\beqa}{\begin{eqnarray}}
\newcommand{\eeqa}{\end{eqnarray}}
\newcommand{\beqan}{\begin{eqnarray*}}
\newcommand{\eeqan}{\end{eqnarray*}}

\renewcommand{\epsilon}{\varepsilon}
\renewcommand{\leq}{\leqslant}

\renewcommand{\hat}{\widehat}

\newcommand{\ra}{\rightarrow}

\usepackage{times}

\begin{document}

\title{Differentially Private Reinforcement Learning with Linear Function Approximation}

\author{%
Xingyu Zhou\thanks{Wayne State University. Email: \texttt{xingyu.zhou@wayne.edu}}}

\date{}

\maketitle

\begin{abstract}
Motivated by the wide adoption of reinforcement learning (RL) in real-world personalized services, where users' sensitive and private information needs to be protected, we study regret minimization in finite horizon Markov decision processes (MDPs) under the constraints of differential privacy (DP).
Compared to existing private RL algorithms that work only on tabular finite-state, finite-actions MDPs, we take the first step towards privacy-preserving learning in MDPs with large state and action spaces. Specifically, we consider MDPs with linear function approximation (in particular linear mixture MDPs) under the notion of joint differential privacy (JDP), where the RL agent is responsible for protecting users' sensitive data.
We design two private RL algorithms that are based on value iteration and policy optimization, respectively, and show that they enjoy sub-linear regret performance while guaranteeing privacy protection. Moreover, the regret bounds are independent of the number of states, and scale at most logarithmically with the number of actions, making the algorithms suitable for privacy protection in nowadays large scale  personalized services. Our results are achieved via a general procedure for learning in linear mixture MDPs under \emph{changing regularizers}, which not only generalizes previous results for non-private learning, but also serves as a building block for general private reinforcement learning.
\end{abstract}

\section{Introduction}
Reinforcement learning (RL) is a control-theoretic protocol, which adaptively learns to make sequential decisions in an unknown environment through trial and error. 
RL has shown to have significant success for delivering a wide variety of personalized services, including online news and advertisement recommendation~\citep{li2010contextual}, 
medical treatment design~\citep{zhao2009reinforcement}, natural language processing~\citep{sharma2017literature}, and social robot~\citep{gordon2016affective}. 
In these applications, an RL agent improves its personalization algorithm by interacting with users to maximize the reward. In particular, in each round, the RL agent offers an action based on the user's current state, and then receives the feedback from the user (e.g., next state and reward, etc.). This feedback is used by the agent to learn the unknown environment and improve its action selection strategy.

However, in most practical scenarios,  the feedback from the users often encodes their sensitive information. Let us consider the user-agent interactions in Google Keyboard (Gboard) application. Here, each user is associated with a sequence of interactions with GBoard, e.g., the complete input of a sentence (or several sentences) can be viewed as an episode in the RL protocol. For each step in the episode, the state is the current typed words, and the action recommended by the GBoard agent is the most promising next words. The reward generated by the user depends on whether or not the user has adopted the offered next words as their actual input. It is easy to see that in this case each user's states and rewards contain her sensitive information. 
Similarly, in a personalized healthcare setting, the states of a patient include personal information such as age, gender, height, weight, state of the treatment etc. 
Another intriguing example is the social robot for second language education of children~\citep{gordon2016affective}. The states include facial expressions, and the rewards contain whether they have passed the quiz. In all of these RL-based personalized applications, users may not want any of their sensitive information to be inferred by adversaries. 
This directly results in an increasing concern about privacy protection in nowadays personalized services. To be more specific, although a user might be willing to share her own information to the agent to obtain a better tailored service, she would not like to allow third parties to infer her private information from the output of the learning algorithm. For example, in the GBoard example, one would like to ensure that an adversary with arbitrary side knowledge and computation power cannot infer her sensitive information by looking at the recommended words to other users.

To this end, \emph{differential privacy} (DP) \citep{dwork2008differential} has become a standard mechanism for designing interactive learning algorithms under a rigorous privacy guarantee for individual data. Most of the previous works on differentially private learning under partial feedback focus on the simpler bandit setting (i.e., no notion of states) \citep{tossou2016algorithms,tossou2017achieving,basu2019differential,mishra2015nearly,zhou2020local,shariff2018differentially}. For the general RL problem, there are only a few works that consider differential privacy \citep{vietri2020private,garcelon2020local}. More importantly, only the \emph{tabular MDPs} are considered in these works where the states and actions are discrete and finite, and the
value function is represented by a table. However, in real-world applications mentioned above, the number of states and actions are often very large and can even be infinite. As a result, the existing private RL algorithms could even fail and moreover their regret bounds~\citep{vietri2020private,garcelon2020local} become quite loose in this case since they are polynomial functions in terms of the number of states and actions. Thus, one fundamental question that remains elusive is: \emph{Is it possible to design private RL algorithms under large state and action spaces with tight regret guarantees?} Furthermore, existing regret bounds in both~\citep{vietri2020private,garcelon2020local} are only for value-based RL algorithms, and a study on policy-based private RL algorithms remains open. Recently, policy optimization (PO) has seen great success in many real-world applications, especially when coupled with function approximations~\citep{silver2017mastering,duan2016benchmarking,wang2018deep}, and a variety of PO based algorithms have been proposed~\citep{williams1992simple,kakade2001natural,schulman2015trust,schulman2017proximal,konda2000actor}. The theoretical understandings of PO have also been studied in both computational (i.e., convergence) perspective~\citep{liu2019neural,wang2019neural} and statistical (i.e., regret) perspective~\citep{cai2020provably,efroni2020optimistic}. Unfortunately, all of these algorithm are non-private and thus a direct application of them on the above personalized services may lead to privacy concerns. The recent work~\citep{sayakPO} takes the first step towards private policy optimization in the tabular setting. Thus, another fundamental question is: \emph{Can we also design sample-efficient policy-based private RL algorithms with linear function approximations?} 

In this paper, we take a systematic approach to affirmatively answer the above two fundamental questions. In particular, our contributions can be summarized as follows.


\begin{itemize}
    \item We take the first step towards designing privacy preserving RL algorithms with rigorous statistical guarantees (i.e., regrets) in MDPs with large state and action spaces. Specifically, under the linear function approximations (i.e., linear mixture MDPs), we design both value iteration and policy optimization based optimistic private RL algorithms and show that they enjoy sublinear regret in total number of steps while guaranteeing \emph{joint differential privacy} (JDP), which, informally, implies that sensitive information about a given user is protected even if an adversary has access to the actions prescribed to all other users.
    \item The regret bound of the value-based algorithm -- \texttt{Private-LinOpt-VI} -- is independent of the number of states and actions whereas the regret of the policy-based algorithm -- \texttt{Private-LinOpt-PO} -- depends at most logarithmically on the number of actions making them suitable for large scale MDPs. Moreover, the cost of privacy under $(\epsilon,\delta)$-JDP for both \texttt{Private-LinOpt-VI} and \texttt{Private-LinOpt-PO} scale only with $\frac{\log(1/\delta)^{1/4}}{\epsilon^{1/2}}$ given the permitted privacy budgets $\epsilon > 0$ and $\delta \in (0,1]$. 
    \item The above results are achieved via a general procedure for learning in linear mixture MDPs under \emph{changing regularizers}, which not only generalizes previous results for non-private learning for both valued-based and policy-based RL, but also serves as a building block for general private RL under various privacy schemes and different privacy guarantees. 
    In particular, as opposed to the existing works, where same regularizers are used in every episode, our approach yields robust parameter estimates and confidence sets against possible adversarial attacks, paving the way for protecting data privacy.
    
\end{itemize}

\subsection{Related Work}
Beside the papers mentioned above, there are other related work on differentially private online learning~\citep{guha2013nearly,agarwal2017price} and multi-armed bandits~\citep{tossou2017achieving,hu2021optimal,sajed2019optimal}. In the linear bandit setting with contextual information,~\citep{shariff2018differentially} shows an impossibility result, i.e., no algorithm can achieve a standard $(\epsilon,\delta)$-DP privacy guarantee while guaranteeing a sublinear regret and thus the relaxed notion of JDP is considered in their paper. This result also implies that standard DP is incompatible with regret minimization in RL. Recently, another variant of DP, called \emph{local differential privacy} (LDP)~\citep{duchi2013local} has gained increasing popularity in personalized services due to its stronger privacy protection. It has been studied in various bandit settings recently~\citep{ren2020multi,zheng2020locally,zhou2020local,dubey2021no}. Under LDP, each user's raw data is directly protected before being sent to the learning agent. Thus, the learning agent only has access to privatized data to train its algorithm, which often leads to a worse regret guarantee compared to DP or JDP. Note that although our main focus is on private RL under JDP constraints, our proposed framework provides a straightforward way to design private RL algorithms under LDP constraints.

In the context of value-based tabular MDPs,~\citep{vietri2020private} gives the first regret bound under JDP while~\citep{garcelon2020local} presents the first regret bound under LDP constraint. As mentioned before, both bounds become loose in the presence of large state and action spaces. In addition to regret minimization,~\citep{balle2016differentially} considers private policy evaluation with linear function approximation. For MDPs with continuous state spaces, \citep{wang2019privacy} proposes a variant of $Q$-learning to protect \emph{only} the rewards information by directly injecting noise into value functions. Recently, a distributed actor-critic RL algorithm under LDP is proposed in~\citep{ono2020locally} but without any regret guarantee. While there are recent advances in regret guarantees for policy optimization~\citep{cai2020provably,efroni2020optimistic,zhou2021nearly,he2021nearly}, we are not aware of any existing work on private policy optimization. 

\textbf{Concurrent and independent work:} There are two concurrent and independent works~\citep{liao2021locally} and~\citep{luyo2021differentially}\footnote{While we upload to Arxiv after our acceptatnce, which is later than both works, we submit our work before Oct. 20, 2021, which is the fall deadline of ACM Sigmetrics conference.}. 
In particular, \cite{liao2021locally} also studied differentially private reinforcement learning under linear mixture MDPs. In particular, they focused on local differential privacy (LDP) and proposed a private variant of UCRL-VTR~\citep{ayoub2020model}. In contrast, we mainly focus on joint differential privacy (JDP) and investigated private variants for both value-based (i.e., UCRL-VTR) and policy-based (i.e., OPPO) algorithms. 
As already highlighted in our paper, although our main focus is JDP, our established framework (relying a general {\reg} and general regret bounds) also provides a convenient way to obtain the regret bounds under LDP constraints as in~\citep{liao2021locally}, following the same modifications in our previous work~\citep{sayakPO} (see Remarks~\ref{rem:ldp} and~\ref{rem:d}). The key difference compared to the JDP case is that now the {\priv} is placed at each local user and in total $O(K)$ private noise rather than $O(\log K)$ is added since each {\priv} at each user adds one independent noise. Nevertheless, for a direct and detailed treatment of privacy guarantee and regret bounds under LDP constraints, we refer readers to~\citep{liao2021locally}.
Finally, in addition to a regret upper bound,~\citep{liao2021locally} also presented a lower bound under LDP constraints.

In the concurrent work~\citep{luyo2021differentially}, the authors also proposed a unified analysis for private UCRL-VTR under both JDP and LDP constraints, which shares a similar high-level idea as ours but more detailed results on the LDP part is provided in~\citep{luyo2021differentially}. While we further study private policy-based algorithms with linear function approximation,~\cite{luyo2021differentially} further investigates private model-free algorithms and highlights interesting and unique issues in the model-free setting. 

\section{Problem Formulation and Preliminaries}

In this section, we present the basics of episodic Markov decision processes with linear function approximations and introduce various notions of differential privacy in reinforcement learning.

\subsection{Learning Problem and Linear Function Approximations}
In this paper, we consider an episodic MDP $(\mathcal{S},\mathcal{A},H, (\mathbb{P}_h)_{h=1}^H, (r_h)_{h=1}^H)$, where $\mathcal{S}$ and $\mathcal{A}$ are state and action spaces, $H$ is the length of each episode, $\mathbb{P}_h(s'|s,a)$ is the probability of transitioning to state $s'$ from state $s$ provided action $a$ is taken at step $h$ and $r_h(s,a)$ is the mean of the reward distribution at step $h$ supported on $[0,1]$.
The actions are chosen following some policy $\pi=(\pi_h)_{h=1}^{H}$, where each $\pi_h$ is a mapping from the state space $\mathcal{S}$ into a probability distribution over the action space $\mathcal{A}$, i.e. $\pi_h(a|s)\ge 0$ and $\sum_{a \in \mathcal{A}}\pi_h(a|s)=1$ for all $s \in \mathcal{S}$. The agent would like to find a policy $\pi$ that maximizes the long term expected reward starting from every state $s \in \mathcal{S}$ and every step $h \in [H]$, defined as
\begin{align*}
    V^{\pi}_h(s):= \mathbb{E}\left[\sum\nolimits_{h'=h}^H r_{h'}(s_{h'}, a_h') \mid s_h = s,\pi\right],
    \end{align*}
where the expectation is with respect to the randomness of
the transition kernel and the policy.
We call $V_{h}^{\pi}$ the value function of policy $\pi$ at step $h$. Now, defining the $Q$-function of policy $\pi$ at step $h$ as
\begin{align*}
    &Q^{\pi}_h(s,a):= \mathbb{E}\left[\sum\nolimits_{h'=h}^H r_{h'}(s_{h'}, a_h')) \mid s_h = s, a_h = a,\pi\right],
    \end{align*}
we obtain $Q^{\pi}_h(s,a) = r_h(s,a)+ \sum_{s' \in \mathcal{S}}V^{\pi}_{h+1}(s') P_h(s'|s,a)$ and $V_h^\pi(s)=\sum_{a \in \mathcal{A}} Q^{\pi}_h(s,a) \pi_h(a|s)$.  
 
A policy $\pi^*$ is said to be optimal if it maximizes the value for all states $s$ and step $h$ simultaneously, and the corresponding optimal value function is denoted by $V^*_{h}(s)=\max_{\pi \in \Pi}V^{\pi}_{h}(s)$ for all $h \in [H]$, where $\Pi$ is the set of all non-stationary policies. The agent interacts with the environment for $K$ episodes to learn the unknown transition probabilities $\mathbb{P}_h(s'|s,a)$ and mean rewards $ r_h(s,a)$, and thus, in turn, the optimal policy $\pi^*$.
At each episode $k$, the agent chooses a policy $\pi^k=(\pi_h^k)_{h=1}^{H}$ and samples a trajectory $\lbrace s_1^k,a_1^k,r_1^k,\ldots,s_H^k,a_H^k,r_H^k,s_{H+1}^k \rbrace$ by interacting with the MDP using this policy. Here, at a given step $h$, $s_h^k$ denotes the state of the MDP, $a_h^k \sim \pi_h^k(\cdot|s_h^k)$ denotes the action taken by the agent, $r_h^k \in [0,1]$ denotes the (random) reward collected by the agent with the mean value $r_h(s_h^k,a_h^k)$ and $s_{h+1}^k\sim \mathbb{P}_h(\cdot | s_h^k,a_h^k)$ denotes the next state. The initial state $s_1^k$ is assumed to be fixed and history independent. We measure performance of the agent by the cumulative (pseudo) regret accumulated
over $K$ episodes, defined as
\begin{align*}
    R(T) := \sum\nolimits_{k=1}^K\left[V_1^{*}(s_1^k) -  V_1^{\pi^k}(s_1^k)\right],
\end{align*}
where $T=KH$ denotes the total number of steps. We seek
algorithms with regret that is sublinear in $T$, which
demonstrates the agent’s ability to act near optimally.

In this paper, we consider episodic MDPs with linear function approximations that allows us to work with large state and action spaces. Specifically, we consider the setting where both the transition dynamics and reward functions are linear with respect to some corresponding feature maps. This setting is often called \emph{linear mixture MDPs} (a.k.a., linear kennel MDPs), see~\citep{ayoub2020model,zhou2021provably} for various examples of linear mixture MDPs including the tabular MDPs.


\begin{assumption}[Linear mixture MDPs]
\label{def:lin}
We assume that the MDP $(\mathcal{S},\mathcal{A},H, \{\mathbb{P}_h\}_{h=1}^H, \{r_h\}_{h=1}^H)$ is a linear mixture MDP with a \emph{known} transition feature mapping $\bp:\mathcal{S}\times \mathcal{A}\times \mathcal{S} \to \Real^{d_1}$ and a \emph{known} value feature mapping $\vp: \mathcal{A}\times \mathcal{S} \to \Real^{d_2}$. That is, for any $h \in [H]$, there exists an unknown vector $\bt_{h,p} \in \mathbb{R}^{d_1}$ with $\norm{\bt_{h,p}}_2 \le \sqrt{d_1}$ such that 
\begin{align*}
    \mathbb{P}(s'| s,a) = \inner{\bp(s,a,s')}{\bt_{h,p}}
\end{align*}
for any $(s,a,s') \in \mathcal{S}\times \mathcal{A}\times \mathcal{S}$, and there exists an unknown vector $\bt_{h,r} \in \Real^{d_2}$ with $\norm{\bt_{h,r}}_2 \le \sqrt{d_2}$ such that 
\begin{align*}
   r_h(s,a) = \inner{\vp(s,a)}{\bt_{h,r}}
\end{align*}
for any $(s,a) \in \mathcal{S}\times \mathcal{A}$. Moreover, we assume $ \norm{\vp(s,a)} \le 1$ for any $(s,a) \in \mathcal{S} \times \mathcal{A}$ and  let $\phi(s,a) := \int_{s'} \bp(s' | s,a)V(s')\, ds'$ and assume that for any function $V: \mathcal{S} \to [0,H]$, $\norm{\phi(s,a)}_2 \le \sqrt{d_1}H$ for any $(s,a) \in \mathcal{S} \times \mathcal{A}$, and $\max\{d_1,d_2\} \le d$.
\end{assumption}

The class of linear mixture MDPs has been adopted in a series of recent (non-private) RL works with linear function approximation including both value-based algorithms~\citep{jia2020model,ayoub2020model,zhou2021provably,zhou2021nearly,he2021logarithmic} and policy-based algorithms~\citep{cai2020provably,ding2021provably}. Another line of works with linear function approximation adopt \emph{linear MDP} model~\citep{jin2020provably,yang2019sample}, which assumes that $ \mathbb{P}(\cdot| s,a) = \inner{\tau(s,a)}{\mu(\cdot)}$ where $\tau(s,a)$ is a known feature mapping and $\mu(\cdot)$ is an unknown measure. Note that the two commonly used linear models are different and one cannot be recovered by the other, see a detailed discussions in~\citep{zhou2021provably}.

\subsection{Differential Privacy in Episodic RL}

In RL-based personalized services, it is common to view each episode $k \in [K]$ as a
trajectory associated to a specific user (cf. Google GBoard application). To this end, we let $U_K = (u_1, \ldots,u_K) \in \cU^K$ to denote a sequence of $K$ unique\footnote{Uniqueness is assumed for simplicity, as for a returning user one can group her with her previous occurrences and use the notion of group privacy to provide tight privacy guarantee.} users participating in the private RL protocol with an RL agent $\cM$, where $\cU$ is the set of all users. {{Formally speaking, as in~\citep{vietri2020private,garcelon2020local}, a user can be captured with a tree of depth $H$ encoding the state and reward responses she would give to all the $A^H$ possible sequences of actions the agent can choose. During the learning process, user $k$ for episode $k$ is identified by the reward and state responses $(r_h^k,s_{h+1}^k)_{h=1}^{H}$ she gives to the actions $( a_{h}^k)_{h=1 }^{H}$ chosen by the agent under policy $\pi^k$}}. 
We let $\cM(U_k)=(a_1^1,\ldots,a_H^K) \in \cA^{KH}$ to denote the set of all actions chosen by the agent $\cM$ when interacting with the user sequence $U_k$.
Informally, we will be interested in (centralized) randomized mechanisms (in this case, RL agents) $\cM$ 
so that the knowledge of the output $\cM(U_k)$ and all but the $k$-th user $u_k$ does not reveal `much' information about $u_k$. 
We formalize this in the following definition.
\begin{definition}[Differential Privacy (DP)]\label{def:DP}
For any $\epsilon \ge 0$ and $\delta \in [0,1]$, a mechanism $\cM : \cU^K \ra \cA^{KH}$ is $(\epsilon,\delta)$-differentially private if for all $U_K,U'_K \in \cU^K$ differing on a single user and for all subset of actions $\cA_0 \subset \cA^{KH}$, 
\beqn
\prob{\cM(U_K)\in \cA_0} \le \exp(\epsilon) \prob{\cM(U'_K)\in \cA_0} + \delta~.
\eeqn
\end{definition}
This is a direct adaptation of the classic notion of differential privacy \citep{dwork2014algorithmic}. 
{{However, we need to relax this definition for our purpose, because standard DP requires that changing the states (or reward response) of a user, the recommended actions to this user will not change too much, which is in conflict with personalization.}} In fact, even in the simpler linear contextual bandit setting, a negative result says that no algorithm can achieve a sublinear regret while guaranteeing standard DP~\citep{shariff2018differentially}.
Hence, to preserve the privacy of individual users while obtaining a nontrivial regret in our RL setting, as in previous works~\citep{shariff2018differentially,vietri2020private}, we consider the notion of \emph{joint differential privacy} (JDP) \citep{kearns2014mechanism}, which requires that simultaneously for all user $u_k$, the
joint distribution of the actions recommended to all users other than $u_k$ be differentially private with respect to the user $u_k$.
It weakens the constraint of DP only in that
the actions suggested specifically to $u_k$ may be sensitive in her type (state and reward responses), which is indeed the intrinsic requirement of personalization services. However, JDP is
still a very strong definition since it protects $u_k$ from any arbitrary collusion of other users against her, so long as she does
not herself make the actions suggested to her public.
To this end, we let $\cM_{-k}(U_K):=\cM(U_K)\setminus ( a_{h}^k)_{h=1}^{H}$ to denote all the actions chosen by the agent $\cM$ excluding those recommended to $u_k$ and formally define JDP as follows.
\begin{definition}[Joint Differential Privacy (JDP)]\label{def:JDP}
For any $\epsilon \!\ge\! 0$ and $\delta \in [0,1]$,
a mechanism $\cM \!:\! \cU^K \!\ra\! \cA^{KH}$ is $(\epsilon,\delta)$-joint differentially private if for all $k \!\in\! [K]$, for all user sequences $U_K,U'_K \!\in\! \cU^K$ differing only on the $k$-th user and for all set of actions $\cA_{-k} \!\subset\! \cA^{(K-1)H}$ given to all but the $k$-th user, 
\beqn
\prob{\cM_{-k}(U_K)\in \cA_{-k}} \le \exp(\epsilon) \prob{\cM_{-k}(U'_K)\in \cA_{-k}} + \delta.
\eeqn
\end{definition}
{Intuitively speaking, JDP in RL means that upon changing one of the participating users (say user $k$), the information observed by the other $K-1$
participants will not change substantially.} JDP has been used extensively in private mechanism design \citep{kearns2014mechanism}, in private matching and allocation problems \citep{hsu2016private}, in designing privacy-preserving algorithms
for linear contextual bandits \citep{shariff2018differentially}, linear quadratic systems~\citep{chowdhury2021adaptive} and it has been introduced in private tabular RL by \citet{vietri2020private}.

In this work, we mainly focus on algorithms that are $(\epsilon,\delta)$-JDP to illustrate the power of our proposed general framework for private RL algorithms. However, in addition to JDP guarantees, our framework also easily carries out to RL with stringer notion of privacy protection, called the \emph{local differential privacy} (LDP)~\citep{duchi2013local}, which is often applied to scenarios where the users may not even be willing to share it's data with the agent directly. In this setting, each user is assumed to have her own privacy mechanism that can do randomized mapping on its data to guarantee privacy. To this end, we denote by $X$ a trajectory $(s_h,a_h,r_h,s_{h+1})_{h=1}^{H}$ and by $\cX$ the set of all possible trajectories. We write $\cM'(X)$ to denote the privatized trajectory generated by a (local) randomized mechanism $\cM'$. With this notation, we now formally define LDP for RL protocol.
\begin{definition}[Local Differential Privacy (LDP)]\label{def:LDP}
For any $\epsilon \ge 0$ and $\delta \in [0,1]$, a mechanism $\cM'$ is $(\epsilon,\delta)$-local differentially private if for all trajectories $X,X' \in \cX$ and for all possible subsets $\cE_0 \subset \lbrace \cM'(X)|X \in \cX\rbrace$,
\beqn
\prob{\cM'(X)\in \cE_0} \le \exp(\epsilon) \prob{\cM'(X')\in \cE_0} + \delta .
\eeqn
\end{definition}
LDP ensures that if any adversary (can be the RL agent itself) observes the output of the privacy mechanism $\cM'$ for two different trajectories, then it is statistically difficult for it to guess which output is from which trajectory. This has been used extensively in multi-armed bandits \citep{zheng2020locally,ren2020multi}, and introduced in private tabular RL by \citet{garcelon2020local}.

\section{Confidence Bounds with Changing Regularizers}
\label{sec:conf}
In this section, we focus on the key component of both algorithm design and regret analysis under linear mixture MDPs for both value-based and policy-based algorithms. That is, the confidence bounds over the regularized least-squared estimators for unknown $\theta_{h,p}$ and $\theta_{h,r}$. In particular, we consider a general setting with \emph{changing regularizers}, which not only allows us to generalize previous non-private ones, but more importantly servers a general procedure for designing private RL algorithms with linear function approximations in later sections. 

To begin with, let's briefly introduce the intuitions behind the regularized linear regression problem under linear mixture MDPs (cf.~\citep{ayoub2020model,zhou2021nearly}). In particular, for any step $h$ at episode $k$, the following equation holds based on the linear mixture MDP model assumption in Assumption~\ref{def:lin}
\begin{align*}
    [\mathbb{P}_hV_{h+1}^k](s_h^k, a_h^k) = \inner{\int_{s'} \bp(s' | s_h^k,a_h^k)V_{h+1}^k(s')\, ds'}{\theta_{h,p}} = \inner{\phi_h^k(s_h^k, a_h^k)}{\theta_{h,p}},
\end{align*}
where $\{V_h^k(\cdot)\}_{(k,h) \in [K]\times [H]}$ is any value function maintained based on data collected before episode $k$, and $\phi_h^k(\cdot, \cdot):=\int_{s'} \bp(s' | \cdot,\cdot)V_{h+1}^k(s')\, ds'$.  Therefore, learning the true $\theta_{h,p}$ could be translated into solving a `linear bandit' problem where the feature is $\phi_h^k(s_h^k,a_h^k)$ and the noise is $V_{h+1}^k(s_{h+1}^k) - [\mathbb{P}_hV_{h+1}^k](s_h^k, a_h^k) $. To this end, one can solve the following regularized linear regression problem:
\begin{align*}
    {\bt}_{h,p}^k &:= \argmin_{\bt \in \Real^{d_1}} \norm{\bt}_{Z_{h,p}^k} + \sum_{j=1}^{k-1}[V_{h+1}^j(s_{h+1}^j)  - \inner{\phi_h^k(s_h^j,a_h^j)}{\bt} ]^2\\
    &=[\underbrace{Z_{h,p}^k + \sum_{j=1}^{k-1} \phi_h^k(s_h^j,a_h^j)\phi_h^k(s_h^j,a_h^j)^{\top}}_{\mathcal{T}_{1,p}} ]^{-1} [ \underbrace{\sum_{j=1}^{k-1} \phi_h^k(s_h^j,a_h^j) V_{h+1}^j(s_{h+1}^j)}_{\mathcal{T}_{2,p}}],
\end{align*}
where $Z_{h,p}^k \in \mathbb{R}^{d_1 \times d_1}$ is possibly changing positive definite matrix.  Setting $\Lambda_{h,p}^k= \mathcal{T}_{1,p}$ and $u_{h,p}^k = \mathcal{T}_{2,p} + z_{h,p}^k$ for some $z_{h,p}^k \in \mathbb{R}^{d_1}$ , we can define a new estimator $\hat{\bt}_{h,p}^k$ under changing regularizers $\{Z_{h,p}^k\}_{k \in [K], h\in [H]}$ and $\{z_{h,p}^k\}_{k \in [K], h\in [H]}$ as
\begin{align}
\label{eq:p_est}
    \hat{\bt}_{h,p}^k = [\Lambda_{h,p}^k]^{-1}u_{h,p}^k.
\end{align}
In particular, choosing $Z_{h,p}^k = \lambda I$ and $z_{h,p}^k = 0$ for all $h, k$ reduces to the standard  non-private value-based algorithms~\citep{ayoub2020model,zhou2021nearly,cai2020provably} and policy-based algorithm~\citep{cai2020provably,ding2021provably} under linear mixture MDPs. 
{Here, $Z_{h,p}^k$ serves as a general regularizer while  $z_{h,p}^k$ can be used to account for the noisy version of $\mathcal{T}_{2,p}$ where the noise could arise from approximation error of $\phi_h^k$, or, from the added Gaussian noise in DP mechanisms in later sections.}

In this paper, we also consider a random reward as opposed to a deterministic reward in previous works. To this end, for the estimator of $\theta_{h,r}$, we consider the minimizer of the following linear regression problem with a positive definite matrix $Z_{h,r}^k$.
\begin{align*}
    {\bt}_{h,r}^k &:= \argmin_{\bt \in \Real^{d_2}} \norm{\bt}_{Z_{h,r}^k} + \sum_{j=1}^{k-1}[r_h^j(s_h^j, a_h^j)  - \inner{\vp(s_h^j,a_h^j)}{\bt} ]^2\\
    &=[\underbrace{Z_{h,r}^k + \sum_{j=1}^{k-1} \vp(s_h^j,a_h^j)\vp(s_h^j,a_h^j)^{\top}}_{\mathcal{T}_{1,r}} ]^{-1} [ \underbrace{\sum_{j=1}^{k-1} \vp(s_h^j,a_h^j) r_h^j(s_h^j,a_h^j)}_{\mathcal{T}_{2,r}}].
\end{align*}
Setting $\Lambda_{h,r}^k= \mathcal{T}_{1,r}$ and $u_{h,r}^k = \mathcal{T}_{2,r} + z_{h,r}^k$ for some $z_{h,r}^k \in \mathbb{R}^{d_2}$, we can define a new estimator $\hat{\bt}_{h,r}^k$ under changing regularizers $\{Z_{h,r}^k\}_{k \in [K], h\in [H]}$ and $\{z_{h,r}^k\}_{k \in [K], h\in [H]}$ as
\begin{align}
\label{eq:r_est}
    \hat{\bt}_{h,r}^k = [\Lambda_{h,r}^k]^{-1}u_{h,r}^k.
\end{align}

One key question now is to guarantee that the two estimators $\hat{\theta}_{h,p}^k$ and $\hat{\theta}_{h,r}^k$ are close to the true parameters $\theta_{h,p}$ and $\theta_{h,r}$ with properly constructed ellipsoids. That is, even under changing regularizers, both $\theta_{h,p}$ and $\theta_{h,r}$ lie in the following ellipsoids with probability at least $1-\alpha$ for any $\alpha \in (0,1]$, respectively
\begin{align*}
    \mathcal{C}_{h,p}^k = \left\{\theta:  \norm{\theta - \hat{\bt}_{h,p}^k}_{\Lambda_{h,p}^k}  \le \beta_p(\alpha)\right\} \text{ and }  \mathcal{C}_{h,r}^k = \left\{\theta:  \norm{\theta - \hat{\bt}_{h,r}^k}_{\Lambda_{h,r}^k}  \le \beta_r(\alpha)\right\},
\end{align*}
where $\beta_p(\alpha)$ and $\beta_r(\alpha)$ are properly chosen parameters to balance between exploration and exploitation. Throughout this paper, we will make the following assumptions on the sequence of regularizers $\{Z_{h,p}^k\}_{k \in [K], h\in [H]}$, $\{z_{h,p}^k\}_{k \in [K], h\in [H]}$, $\{Z_{h,r}^k\}_{k \in [K], h\in [H]}$ and $\{z_{h,r}^k\}_{k \in [K], h\in [H]}$. 

\begin{assumption}[Regularity]
\label{ass:regularizer}
For any $\alpha \in (0,1]$, there exist constants $\lambda_{\max,p}, \lambda_{\max,r}$, $\lambda_{\min,p}, \lambda_{\min,r}$ and $\nu_p, \nu_r$ depending on $\alpha$ such that, with probability at least $1-\alpha$, for all $k \in [K]$ and $h \in [H]$, the following relations hold simultaneously
\begin{align*}
    \norm{Z_{h,p}^k } \leq \lambda_{\max,p}, \;\; \norm{(Z_{h,p}^k)^{-1}} \leq 1/\lambda_{\min,p} \;\; \text{and}\;\; \norm{z_{h,p}^k}_{(Z_{h,p}^k)^{-1}} \leq \nu_{p},
\end{align*}
and
\begin{align*}
    \norm{Z_{h,r}^k } \leq \lambda_{\max,r}, \;\; \norm{(Z_{h,r}^k)^{-1}} \leq 1/\lambda_{\min,r} \;\; \text{and}\;\; \norm{z_{h,r}^k}_{(Z_{h,r}^k)^{-1}} \leq \nu_{r}.
\end{align*}
\end{assumption}

With this regularity assumption, we have the following high-probability confidence bounds over the two estimators, i.e., $\hat{\theta}_{h,p}^k$ in~\eqref{eq:p_est} and $\hat{\theta}_{h,r}^k$ in~\eqref{eq:r_est}. 

\begin{lemma}[Concentration under changing regularizers]
\label{lem:CI}
Let Assumptions~\ref{def:lin} and~\ref{ass:regularizer} hold and $V_h^k(\cdot)$ in~\eqref{eq:p_est} be within $[0,H]$ for all $h \in [H], k \in [K]$. Then, for any $\alpha \in (0,1]$, with probability at least $1-3\alpha$, the following relations hold simultaneously for all $k \in [K]$ and $h \in [H]$
\begin{align*}
    \norm{\bt_{h,p} - \hat{\bt}_{h,p}^k}_{\Lambda_{h,p}^k} \le \beta_p(\alpha), \text{ and }  \norm{\bt_{h,r} - \hat{\bt}_{h,r}^k}_{\Lambda_{h,r}^k} \le \beta_r(\alpha),
\end{align*}
where 
\begin{align*}
    &\beta_{p}(\alpha) = \frac{H}{2}\sqrt{2\ln\left(\frac{H}{\alpha}\right) +  d_1\ln\left(1 + \frac{KH^2}{\lambda_{\min,p}}\right)} + \sqrt{d_1 \lambda_{\max,p} }+ \nu_{p},\\
    &\beta_{r}(\alpha) = \frac{1}{2}\sqrt{2\ln\left(\frac{H}{\alpha}\right) +  d_2\ln\left(1 + \frac{K}{d_2\lambda_{\min,r}}\right) } + \sqrt{d_2 \lambda_{\max,r} }+ \nu_{r}.
\end{align*}
\end{lemma} 
\begin{proof}
See Appendix~\ref{sec:proofCI}.
\end{proof}

In the next two sections, we will utilize the two confidence bounds to investigate general value-based and policy-based RL algorithms under changing regularizers. These results not only generalize the standard results for fixed regularizers, but allow us to present a general framework for designing private RL algorithms with differential privacy guarantees in Section~\ref{sec:priv}.

\section{Optimistic Value-based Algorithm with Changing Regularizers}

In this section, we propose an optimistic value-based RL algorithm under linear mixture MDPs with changing regularizers. Our algorithm can be regarded as a generalization of  
UCRL-VTR~\citep{ayoub2020model} where fixed regularizers are adopted. This generalization serves a basis for designing private value-based algorithm in Section~\ref{sec:priv}. 

As shown in Algorithm~\ref{alg:UCB-VI}, compared to UCRL-VTR, the main difference of our \texttt{LinOpt-VI-Reg} is the adoption of the {\reg} that properly chooses the four sequences of regularizers $\{Z_{h,p}^k\}_{k \in [K], h\in [H]}$, $\{z_{h,p}^k\}_{k \in [K], h\in [H]}$, $\{Z_{h,r}^k\}_{k \in [K], h\in [H]}$ and $\{z_{h,r}^k\}_{k \in [K], h\in [H]}$. Similar to UCRL-VTR, \texttt{LinOpt-VI-Reg} follows the procedure of optimistic value iteration. Specifically, at each episode $k$, it first computes the two estimators and the corresponding exploration bonus terms based on regularized statistics. Then, it computes the estimated $Q$-function and value function using optimistic Bellman recursion. Next, a greedy policy $\pi^k$ is obtained directly from the estimated $Q$-function. Finally, a trajectory is rolled out by acting the policy $\pi^k$ and then the {\reg} translates these statistics into regularized ones via the way in Section~\ref{sec:conf}, i.e., using the four regularizers $\{Z_{h,p}^k\}_{k \in [K], h\in [H]}$, $\{z_{h,p}^k\}_{k \in [K], h\in [H]}$, $\{Z_{h,r}^k\}_{k \in [K], h\in [H]}$ and $\{z_{h,r}^k\}_{k \in [K], h\in [H]}$. 
The following theorem presents a general regret bound when the {\reg} in \texttt{LinOpt-VI-Reg} satisfies Assumption~\ref{ass:regularizer}.
\begin{algorithm}[t!]
\caption{\texttt{LinOpt-VI-Reg} }
\label{alg:UCB-VI}
\DontPrintSemicolon
\KwIn{Number of episodes $K$, time horizon $H$,  a {\reg}, confidence level $\alpha \in (0,1]$}
Initialize regularized statistics $ {\Lambda}_{h,p}^{1}=0$, $ {\Lambda}_{h,r}^{1}=0$, $ {u}_{h,p}^{1}=0$ and  $ {u}_{h,r}^{1}=0$  \;
\For{$k=1,\ldots, K$}{
Initialize value estimates: $ {V}^k_{H+1}(s) = 0$ $\forall s$  \;
\For{$h=H,H-1,\ldots, 1$}{
    $ \hat{\bt}_{h,p}^k \leftarrow ( {\Lambda}_{h,p}^k)^{-1}  {u}_{h,p}^k$\;  
    $ \hat{\bt}_{h,r}^k \leftarrow ( {\Lambda}_{h,r}^k)^{-1}  {u}_{h,r}^k$\;
    ${\phi}_h^k \leftarrow \int_{s'} \bp(s' | s,a) {V}_{h+1}^k(s')\, ds$ \;
    $ {\Gamma}_{h,p}^{k} \leftarrow \beta_p(\alpha) [\phi_h^k(\cdot,\cdot)^{\top} ( {\Lambda}_{h,p}^k)^{-1} \phi_h^k(\cdot,\cdot) ]^{1/2}$ \;
    $ {\Gamma}_{h,r}^{k} \leftarrow \beta_r(\alpha) [\varphi(\cdot,\cdot)^{\top} ( {\Lambda}_{h,r}^k)^{-1} \varphi(\cdot,\cdot) ]^{1/2}$ \;
    $ {Q}_h^k(\cdot, \cdot) \leftarrow \max\{0, \min\{H-h+1, \vp(\cdot,\cdot)^{\top} \hat{\bt}_{h,r}^k+ \phi_h^k(\cdot,\cdot)^{\top} \hat{\bt}_{h,p}^k + ( {\Gamma}_{h,p}^k +  {\Gamma}_{h,r}^k)(\cdot,\cdot) \}\}$\;
    $ {V}_h^k(\cdot) \leftarrow \max_{a \in \mathcal{A}}  {Q}_h^k(\cdot,a) $
}
$\forall h, \pi_h^k(\cdot) \leftarrow \argmax_{a \in \mathcal{A}}  {Q}_h^k(\cdot,a) $\;
Roll out a  trajectory $(s_1^k, a_1^k, r_1^k,\ldots, s_{H+1}^k)$ by acting the policy $\pi^k = (\pi_h^k)_{h=1}^H$\;
Receive regularized statistics $ {\Lambda}_{h,p}^{k+1}$, $ {\Lambda}_{h,r}^{k+1}$, $ {u}_{h,p}^{k+1}$ and  $ {u}_{h,r}^{k+1}$ from the {\reg}
}
\end{algorithm}

\begin{theorem}
\label{thm:VI}
Let Assumption~\ref{def:lin} hold and the {\reg} in Algorithm~\ref{alg:UCB-VI} satisfy Assumption~\ref{ass:regularizer}. Then, for any $\alpha \in (0,1]$, with probability at least $1-4\alpha$, Algorithm~\ref{alg:UCB-VI} achieves regret 
\begin{align*}
    \mathcal{R}(T) &= O\left(\sqrt{TH^3}\left(\ln \frac{H}{\alpha} + d \kappa_{\lambda_{\min}}\right) + \sqrt{TH d \kappa_{\lambda_{\min}} } \left(\sqrt{d \lambda_{\max}} + \nu\right) \right)
\end{align*}
where $\kappa_{\lambda_{\min}}:=\ln\left(1+KH^2/(\lambda_{\min})\right)$ and $\lambda_{\min} = \min\{\lambda_{\min, p}, \lambda_{\min, r}\}$, $\lambda_{\max} = \max\{\lambda_{\max, p}, \lambda_{\max, r}\}$ and $\nu = \max\{\nu_{p}, \nu_{r}\} $
\end{theorem}

\begin{proof}[Proof Sketch]
Let $\mathcal{E}(\alpha): = \{\theta_{h,p} \in \mathcal{C}_{h,p}^{k} \text{ and }  \theta_{h,r} \in \mathcal{C}_{h,r}^{k}, \forall h \in [H], k \in [K]\}$ denote the `good' event, which by Lemma~\ref{lem:CI}, holds with probability at least $1-3\alpha$. Under this good event, we can first show that $Q_{h}^k$ maintained in Algorithm~\ref{alg:UCB-VI} is optimistic with respect to the optimal $Q$-value $Q^*_h$. With this property and the linear mixture MDPs assumption, the regret can be decomposed as follows. 
\begin{align*}
    \sum_{k=1}^K V_{k,1}(s_1^k) - V_{1}^{\pi^k}(s_1^k) &\le \underbrace{\sum_{k=1}^K \sum_{h=1}^H2\beta_p(\alpha)\min\left\{1, \norm{\phi_h^k(s_h^k,a_h^k)}_{(\Lambda_{h,p}^k)^{-1}}\right\} }_{\mathcal{T}_1}\\
    &+ \underbrace{\sum_{k=1}^K \sum_{h=1}^H2H\beta_r(\alpha)\min\left\{1, \norm{\vp(s_h^k,a_h^k)}_{(\Lambda_{h,r}^k)^{-1}}\right\}}_{\mathcal{T}_2}\\
    &+ \underbrace{\sum_{k=1}^K\sum_{h=1}^H \mathbb{P}_h[V_{h+1}^k - V_h^{\pi^k}](s_h^k,a_h^k) - [V_{h+1}^k - V_h^{\pi^k}](s_{h+1}^k)}_{\mathcal{T}_3}.
\end{align*}
With a properly defined filtration, $\mathcal{T}_3$ is a sum of bounded martingale difference sequence. Therefore, by Azuma-Hoeffding inequality, we have with probability at least $1-\alpha$, $ \mathcal{T}_3 \le H\sqrt{2KH\log(1/\alpha)}$, where we have used boundedenss of value function in Algorithm~\ref{alg:UCB-VI}. To bound $\mathcal{T}_1$ and $\mathcal{T}_2$, we will use the trace-determinant lemma  and elliptical potential lemma (cf.~\citep[Lemmas 10 and 11]{abbasi2011improved}). In particular, by the boundedness of $\norm{\phi_h^k(s_h^k, a_h^k) } \le \sqrt{d_1} H$ and $\norm{\vp(s_h^k,a_h^k)} \le 1$, and the Assumption~\ref{ass:regularizer} on the {\reg}, we obtain that 
\begin{align*}
    &\mathcal{T}_1 \le 2\beta_p(\alpha)\sqrt{KH}\sqrt{2Hd_1\log(1+KH^2/(\lambda_{\min,p}))}\\
    &\mathcal{T}_2 \le 2H\beta_r(\alpha)\sqrt{KH}\sqrt{2Hd_2\log(1+K/(d_2\lambda_{\min,r}))}.
\end{align*}
Finally, plugging in the choices of $\beta_p(\alpha)$ and  $\beta_r(\alpha)$ in Lemma~\ref{lem:CI} and putting everything together, yields the required result. The full proof is given at Section~\ref{sec:proofthmvi}.
\end{proof}

\begin{remark}
This theorem serves as a building block for analyzing private value-based RL algorithm under linear mixture MDPs. In particular, the general procedure is to first instantiate the {\reg} with a specific privacy protection scheme, and then determine the corresponding values of $\lambda_{\max}, \lambda_{\min}, \nu$. Finally, with Theorem~\ref{thm:VI}, one can easily obtain the private regret. One key feature of this result is that it provides a general analytical framework for private value-based algorithms under different private schemes and even different types of privacy guarantees (e.g., JDP and LDP). See more discussions in Section~\ref{sec:priv}.
\end{remark}

\section{Optimistic Policy-based Algorithm with Changing Regularizers}
In this section, we turn to focus on policy-based algorithm under linear mixture MDPs. Our proposed algorithm can be regarded as a bandit-feedback variant of the OPPO algorithm in~\citep{cai2020provably} where the full information of the reward is observed and the regularizer is fixed. 
As in the last section, this  generalization serves as a basis for designing private policy-based algorithm in later section.

As shown in Algorithm~\ref{alg:UCB-PO}, our proposed \texttt{LioOpt-PO-Reg} basically consists of two stages at each episode $k$: \emph{policy evaluation} and \emph{policy improvement}. In the policy evaluation stage, it  uses the Least-Squares Temporal Difference (LSTD)~\citep{boyan1999least,bradtke1996linear} to evaluate policy $\pi^k$ based on previous $k-1$ historical trajectories. In contrast to standard process, \texttt{LinOpt-PO-Reg}  uses regualrized statistics produced by possible changing regularizers at this stage. In particular, it estimates $\mathbb{P}_h V_{h+1}^k$ with $(\phi_h^k)^{\top} \hat{\theta}_{h,p}^k$ and estimates $r_h$  with $\vp^{\top} \hat{\theta}_{h,r}^k$. These two estimates along with two UCB exploration bonuses are used to iteratively compute $Q$-function estimates. The $Q$-estimates are then truncated to the interval $[0,H]$ since the reward is within $[0,1]$. Then, the corresponding value estimates are computed by taking the expectation of $Q$-estimates with respect to the policy. Next, a new trajectory is rolled out by acting the policy $\pi^k$ and the {\reg} translates the collected statistics into regularized ones via the four changing regularizers. Finally, in the policy improvement stage, \texttt{LinOpt-PO-Reg} employs a `soft' update of the current policy $\pi^k$ by following a standard mirror-descent step together with a Kullback-Leibler (KL) divergence proximity term~\citep{cai2020provably}. 
The following theorem presents a general regret bound when the {\reg} in \texttt{LinOpt-PO-Reg} satisfies Assumption~\ref{ass:regularizer}.

\begin{theorem}
\label{thm:PO}
Let Assumption~\ref{def:lin} hold and the {\reg} in Algorithm~\ref{alg:UCB-PO} satisfy Assumption~\ref{ass:regularizer}. Then, for any $\alpha \in (0,1]$, with probability at least $1-4\alpha$, Algorithm~\ref{alg:UCB-PO} achieves regret
   \begin{align*}
    \mathcal{R}(T) = O\left(\sqrt{TH^3}\left(\log \frac{H}{\alpha} + d \kappa_{\lambda_{\min}}\right) +  \sqrt{TH^3 \log |\mathcal{A}|} + \sqrt{TH d \kappa_{\lambda_{\min}} } \left(\sqrt{d \lambda_{\max}} + \nu\right) \right),
\end{align*}
where $\kappa_{\lambda_{\min}}:=\ln\left(1+KH^2/(\lambda_{\min})\right)$ and $\lambda_{\min} = \min\{\lambda_{\min, p}, \lambda_{\min, r}\}$, $\lambda_{\max} = \max\{\lambda_{\max, p}, \lambda_{\max, r}\}$ and $\nu = \max\{\nu_{p}, \nu_{r}\}$. Further, suppose $\sqrt{\log |\mathcal{A}|} = O(\log \frac{H}{\alpha} + d \kappa_{\lambda_{\min}})$, the regret satisfies  
\begin{align*}
    \mathcal{R}(T) = O\left(\sqrt{TH^3}\left(\log \frac{H}{\alpha} + d \kappa_{\lambda_{\min}}\right) + \sqrt{TH d \kappa_{\lambda_{\min}} } \left(\sqrt{d \lambda_{\max}} + \nu\right) \right).
\end{align*}
\end{theorem}
\begin{proof}[Proof Sketch]
As in~\citep{cai2020provably}, our regret analysis starts with the following regret decomposition.
\begin{align*}
\mathcal{R}(T)  &= \sum_{k=1}^K\left(V_1^*(s_1^k) - V_1^{\pi_k} (s_1^k)\right)\\
&= \underbrace{\sum_{k=1}^K \sum_{h=1}^H \mathbb{E}_{\pi^*}\left[ \inner{Q_h^k(s_h^k)}{\pi_h^*(\cdot | s_h)- \pi_h^{\pi_k}(\cdot | s_h)} \mid s_1 = s_1^k\right] }_{(i)} + \underbrace{\mathcal{M}_{K,H,2}}_{(ii)}\\
&+ \underbrace{\sum_{k=1}^K \sum_{h=1}^H \left(\mathbb{E}_{\pi^*}[\iota_h^k(x_h,a_h) \mid s_1 = s_1^k] - \iota_h^k(s_h^k, a_h^k)\right) }_{(iii)},
\end{align*}
where $\iota_h^k:=r_h + \mathbb{P}_h V_{h+1}^k - Q_h^k$ and $\mathcal{M}_{K,H,2}$ is a martingale adapted to a properly defined filtration. Hence, by the boundedness of both $Q$-function and value function and Azuma-Hoeffding inequality, we can establish that with probability at least $1-\alpha$, $(ii) \le 4H \sqrt{2T\log(1/\alpha)}$. For $(i)$, we can bound it by using standard properties of KL-divergence and mirror-descent update along with the fact that $\pi^1$ is a uniform distribution. In particular, we have $(i) \le \sqrt{2H^3T\log |\mathcal{A}|}$. The key part in the proof is to bound $(iii)$. To this end, we will first show that under good event $\mathcal{E}(\alpha)$, we have for all $s,a,k,h$, 
\begin{align}
    2 ( {\Gamma}_{h,p}^k +  {\Gamma}_{h,r}^k)(s,a) \ge -\iota_h^k(s,a) \ge 0.
\end{align}
Therefore, with the above property, we can see that bounding $(iii)$ reduces to bounding the summation of both bonus terms. Similar to the analysis in the last section, we have 
\begin{align*}
    (iii) = O\left(\sqrt{TH^3}\left(\log \frac{H}{\alpha} + d \kappa_{\lambda_{\min}}\right) + \sqrt{TH d \kappa_{\lambda_{\min}} } \left(\sqrt{d \lambda_{\max}} + \nu\right) \right),
\end{align*}
where $\kappa_{\lambda_{\min}}:=\log\left(1+KH^2/(\lambda_{\min})\right)$ and $\lambda_{\min} = \min\{\lambda_{\min, p}, \lambda_{\min, r}\}$, $\lambda_{\max} = \max\{\lambda_{\max, p}, \lambda_{\max, r}\}$ and $\nu = \max\{\nu_{p}, \nu_{r}\} $. Finally, putting the three terms together, yields the required result. The full proof is given at Section~\ref{sec:proofthmpo}.
\end{proof}
\begin{remark}
This theorem provides a clean result to design the first private PO algorithms with linear function approximations in the next section. As in the last section, this result offers a simple approach to designing private PO algorithms under different privacy guarantees, which will be explored in the next section.
\end{remark}

\begin{algorithm}[t!]
\caption{ \texttt{LinOpt-PO-Reg}
}
\label{alg:UCB-PO}
\DontPrintSemicolon
\KwIn{Number of episodes $K$, time horizon $H$, parameter $\eta$, a {\reg}, confidence level $\alpha \in (0,1]$}
Initialize regularized statistics $ {\Lambda}_{h,p}^{1}=0$, $ {\Lambda}_{h,r}^{1}=0$, $ {u}_{h,p}^{1}=0$ and  $ {u}_{h,r}^{1}=0$  \;
Initialize policy $\pi^1 = (\pi_h^1)_{h=1}^H$ be uniform distributions over $\mathcal{A}$\;
\For{$k=1,\ldots, K$}{
Initialize value estimates: $ {V}^k_{H+1}(s) = 0$ $\forall s$  \;
\For{$h=H,H-1,\ldots, 1$}{
    $ \hat{\bt}_{h,p}^k \leftarrow ( {\Lambda}_{h,p}^k)^{-1}  {u}_{h,p}^k$\;  
    $ \hat{\bt}_{h,r}^k \leftarrow ( {\Lambda}_{h,r}^k)^{-1}  {u}_{h,r}^k$\;
    ${\phi}_h^k \leftarrow \int_{s'} \bp(s' | s,a) {V}_{h+1}^k(s')\, ds$ \;
    $ {\Gamma}_{h,p}^{k}(\cdot, \cdot) \leftarrow \beta_p [\phi_h^k(\cdot,\cdot)^{\top} ( {\Lambda}_{h,p}^k)^{-1} \phi_h^k(\cdot,\cdot) ]^{1/2}$ \;
    $ {\Gamma}_{h,r}^{k}(\cdot, \cdot)\leftarrow \beta_r [\varphi(\cdot,\cdot)^{\top} ( {\Lambda}_{h,r}^k)^{-1} \varphi(\cdot,\cdot) ]^{1/2}$ \;
    $ {Q}_h^k(\cdot, \cdot) \leftarrow \max\{0, \min\{H-h+1, \vp(\cdot,\cdot)^{\top} \hat{\bt}_{h,r}^k+ \phi_h^k(\cdot,\cdot)^{\top} \hat{\bt}_{h,p}^k + ( {\Gamma}_{h,p}^k +  {\Gamma}_{h,r}^k)(\cdot,\cdot) \}\}$\;
    $ {V}_h^k(\cdot) \leftarrow \inner{ {Q}_h^k(\cdot,\cdot)}{\pi_h^k(\cdot | \cdot)}_{\mathcal{A}} $
}
Roll out a  trajectory $(s_1^k, a_1^k, r_1^k,\ldots, s_{H+1}^k)$ by acting the policy $\pi^k = (\pi_h^k)_{h=1}^H$\;
Receive regularized statistics $ {\Lambda}_{h,p}^{k+1}$, $ {\Lambda}_{h,r}^{k+1}$, $ {u}_{h,p}^{k+1}$ and  $ {u}_{h,r}^{k+1}$ from the {\reg}\;
Update policy $\forall h$ , $\pi_h^{k+1}(\cdot|\cdot) \propto {\pi}_h^k(\cdot|\cdot)  \exp \{\eta  {Q}_h^k(\cdot, \cdot)\}$

}
\end{algorithm}

\section{Privacy and Regret Guarantees}
\label{sec:priv}
In this section, we shall see how to instantiate the {\reg} in previous two algorithms with a specific privacy protection scheme--called {\priv} hereafter--to guarantee JDP while achieving sublinear regret.

\subsection{{\priv} Design}
The main idea behind the {\priv} is to inject Gaussian noise as the changing regularizers, i.e., $Z_{h,p}^k$ and $Z_{h,r}^k$ are chosen to be Gaussian random matrices  and $z_{h,p}^k$ and $z_{h,r}^k$ are Gaussian random vectors so that the standard Gaussian mechanism~\citep{dwork2014algorithmic} can be applied. These noise are used to privatize the non-private statistics $G_{h,p}^k:=\sum_{j=1}^{k-1} \phi_h^k(s_h^j,a_h^j)\phi_h^k(s_h^j,a_h^j)^{\top}$, $G_{h,r}^k:=\sum_{j=1}^{k-1} \vp(s_h^j,a_h^j)\vp(s_h^j,a_h^j)^{\top}$, $g_{h,p}^k:=\sum_{j=1}^{k-1} \phi_h^k(s_h^j,a_h^j) V_{h+1}^j(s_{h+1}^j)$ and $g_{h,r}^k:=\sum_{j=1}^{k-1} \vp(s_h^j,a_h^j) r_h^j(s_h^j,a_h^j)$, all of which are in the form of sums. This motivates us to consider the \emph{binary counting mechanism}~\citep{chan2011private}, which is a tree-based utility-efficient scheme for releasing private counts. In the following, let's take the design of  $Z_{h,p}^k$ and $z_{h,p}^k$ as an example and the same method applies to  $Z_{h,r}^k$ and $z_{h,r}^k$. We first initialize $2H$ counters, $\mathcal{B}_{h,p,1}$ and $\mathcal{B}_{h,p,2}$ for each $h \in [H]$. In particular, $\mathcal{B}_{h,p,1}$ maintains a binary interval tree in which each node represents a partial sum (P-sum) in the form of $\Sigma_{p,1}[i,j] = \sum_{k=i}^j \phi_h^k(s_h^k,a_h^k)\phi_h^k(s_h^k,a_h^k)^{\top}$ involving the features in episodes $i$ through $j$. Each leaf node represents a P-sum with a single episode (i.e., $i=j$ and hence the tree has $k-1$ leaf nodes at the start of episode $k$), and each interior node represents the range of episodes covered by its children. The same construction applies to $\mathcal{B}_{h,p,2}$ with the P-sum given by $\Sigma_{p,2}[i,j] = \sum_{k=i}^{j} \phi_h^k(s_h^k,a_h^k) V_{h+1}^k(s_{h+1}^k)$. Now, to release the private version of $G_{h,p}^k$ for each $k$ denoted by $\hat{G}_{h,p}^k$, $\mathcal{B}_{h,p,1}$ first release a noisy P-sum $\hat{\Sigma}_{p,1}[i,j]$ corresponding to each node in the tree. Here $\hat{\Sigma}_{p,1}[i,j]$ is obtained by perturbing both $(p,q)$-th and $(q,p)$-th, $1\le p\le q\le d_1$, entries of $\Sigma_{p,1}[i,j]$ with i.i.d Gaussian noise $\mathcal{N}(0,\sigma_{p,1}^2)\footnote{This will ensure symmetric of the P-sums even after adding noise.}$. Then the private $\hat{G}_{h,p}^k$ is computed by summing up the noisy P-sums that uniquely cover the range $[1,k-1]$. {Algorithmically speaking, $\mathcal{B}_{h,p,1}$ is constructed and maintained by simply replacing the input stream of 0/1 bit in Algorithm 2 of~\citep{chan2011private}  with the stream of $\{\phi_h^k(s_h^k,a_h^k)\phi_h^k(s_h^k,a_h^k)^{\top}\}_{k=1}^K$ in our case.}
To guarantee the resultant matrix is still positive definite, $\mathcal{B}_{h,p,1}$ shifts  $\hat{G}_{h,p}^k$ by $\gamma I$ for some $\gamma >0$ specified later. Thus, the final matrix $\Lambda_{h,p}^k$ is the sum of non-private statistics $G_{h,p}^k$ and noise matrix $Z_{h,p}^k$,  where $Z_{h,p}^k = N_{h,p}^k + \gamma I$ and $N_{h,p}^k$ is the total noise added via the summation of noisy P-sums. Similarly, for the counter $\mathcal{B}_{h,p,2}$, the noisy P-sum $\hat{\Sigma}_{p,2}[i,j]$ is generated by perturbing each element of the  $d_1\times 1$ vector with i.i.d Gaussian noise $\mathcal{N}(0,\sigma_{p,2}^2)$. Then, the private $\hat{g}_{h,p}^k$ is computed by summing up the noisy P-sums that uniquely cover the range $[1,k-1]$. As a result, the final vector $u_{h,p}^k$ is the sum of $g_{h,p}^k$ and $z_{h,p}^k$, where $z_{h,p}^k$ is the total noise added in the covering noisy P-sums. We follow the same way to maintain the counters $\mathcal{B}_{h,r,1}$ and $\mathcal{B}_{h,r,2}$ to privatize $G_{h,r}^k$ and $g_{h,r}^k$ and hence obtaining $Z_{h,r}^k$ and $z_{h,r}^k$. Note that, at the end of each episode, the counters only need to store noisy P-sums required for computing private statistics at future episodes, and hence can safely discard P-sums that are no longer needed.

One key problem in differential privacy is to balance between privacy guarantee and utility (i.e., regret in our case). That is, if more noise is added during the process, it becomes easier to guarantee privacy but meanwhile suffers from a worse regret. The distinguishing feature of binary counting mechanism discussed above is the ability to maintain a nice trade-off between privacy and utility, which is demonstrated via the following fundamental properties~\citep{chan2011private}. 
\begin{lemma}[Fundamental properties of binary counting mechanism] 
\label{lem:fact}
Let $m:= \lceil\log K\rceil $. For all the counters $\mathcal{B}_{h,p,1}$, $\mathcal{B}_{h,p,2}$, $\mathcal{B}_{h,r,1}$ and $\mathcal{B}_{h,r,2}$ introduced above, we have 
\begin{enumerate}
    \item[(i)] Each episode appears only in at most $m$ P-sums.
    \item[(ii)] There are at most $m$ P-sums that uniquely cover the range $[1,k-1],~\forall k \in [K]$ .
\end{enumerate}
\end{lemma}

In the following, we will rely on the two properties in Lemma~\ref{lem:fact} to analyze the privacy and regret guarantees, respectively. Before that, we first formally define two private RL algorithms by instantiating the {\reg} in Algorithm~\ref{alg:UCB-VI} and~\ref{alg:UCB-PO} with the {\priv} that consists of the four counters $\mathcal{B}_{h,p,1}$, $\mathcal{B}_{h,p,2}$, $\mathcal{B}_{h,r,1}$ and $\mathcal{B}_{h,r,2}$ introduced above.

\begin{definition}[\texttt{Private-LinOpt-VI}]
\texttt{Private-LinOpt-VI} is defined as an instance of  Algorithm~\ref{alg:UCB-VI} with the {\reg} being instantiated by a {\priv} that consists of the four private counters $\mathcal{B}_{h,p,1}$, $\mathcal{B}_{h,p,2}$, $\mathcal{B}_{h,r,1}$ and $\mathcal{B}_{h,r,2}$.
\end{definition}

\begin{definition}[\texttt{Private-LinOpt-PO}]
\texttt{Private-LinOpt-PO} is defined as an instance of Algorithm~\ref{alg:UCB-PO} with the {\reg} being instantiated by a {\priv} that consists of the four private counters $\mathcal{B}_{h,p,1}$, $\mathcal{B}_{h,p,2}$, $\mathcal{B}_{h,r,1}$ and $\mathcal{B}_{h,r,2}$.
\end{definition}

\subsection{Privacy Guarantee}
In this section, we will show that both private algorithms provide JDP guarantee, which is summarized in the following theorem.

\begin{theorem}
\label{thm:priv}
Let Assumption~\ref{def:lin} hold. For any $\epsilon >0$ and $\delta \in (0,1]$, both \texttt{Private-LinOpt-VI} and \texttt{Private-LinOpt-PO} can guarantee $O(\epsilon,\delta)$-JDP.
\end{theorem}
\begin{proof}
We first show that with proper choices of noise parameter in the four counters (i.e., $\sigma_{p,1}, \sigma_{p,2}, \sigma_{r,1}, \sigma_{r,2}$), all the four counters are $(\epsilon/4, \delta/4)$-DP. We begin with the counter $\mathcal{B}_{h,p,1}$. To this end, we need to determine a global upper bound $\Delta_{p,1}$ over the $L_2$-sensitivity so as to apply Gaussian mechanism~\citep{dwork2014algorithmic}. Here, $\Delta_{p,1}$ encodes the maximum change in the Frobenious norm of each P-sum if the state-action pair of a single episode is changed. By the boundedness assumptions in Assumption~\ref{def:lin}, we have $\norm{\phi_h^k(s,a)} \le \sqrt{d_1}H, \forall s,a,h,k$, and hence $\Delta_{p,1} = d_1 H^2 \le dH^2$. Since the noisy P-sums are obtained via Gaussian noise, we have that each $\hat{\Sigma}_{p,1}[i,j]$ is $(\Delta_{p,1}^2/2\sigma_{p,1}^2)$-CDP~\citep[Proposition 1.6]{bun2016concentrated}. Now, by the property (i) in Lemma~\ref{lem:fact} and the composition property of DP, the whole counter $\mathcal{B}_{h,p,1}$ is $(m\Delta_{p,1}^2/2\sigma_{p,1}^2)$-CDP, and thus, in turn, $\left(\frac{m\Delta_{p,1}^2}{2\sigma_{p,1}^2}\!+\!2\sqrt{\frac{m\Delta_{p,1}^2}{2\sigma_{p,1}^2}\ln\left(\frac{4}{\delta}\right)},\frac{\delta}{4}\right)$-DP for any $\delta \!>\! 0$~\citep[Lemma 3.5]{bun2016concentrated}. Now, setting $\sigma_{p,1}^2 = 32m\Delta_{p,1}^2 \ln(4/\delta)/\epsilon^2$, we can ensure that $\mathcal{B}_{h,p,1}$ is $O(\epsilon/4, \delta/4)$-DP. We then follow the same approach to establish that $\mathcal{B}_{h,p,2}$ is $O(\epsilon/4, \delta/4)$-DP with $\sigma_{p,2}^2 = 32m\Delta_{p,2}^2 \ln(4/\delta)/\epsilon^2$ and $\Delta_{p,2} = \sqrt{d_1}H^2 \le \sqrt{d}H^2$. Similarly, we have $\mathcal{B}_{h,r,1}$ and $\mathcal{B}_{h,r,2}$ are both $O(\epsilon/4, \delta/4)$-DP with $\sigma_{r,1}^2 = \sigma_{r,2}^2 = 32m\Delta^2 \ln(4/\delta)/\epsilon^2$ with  $\Delta = \Delta_{r,1} = \Delta_{r,2}  =  1$. 

To prove the final result, we now use the \emph{billboard lemma}~\citep[Lemma 9]{hsu2016private}, which states that an algorithm is JDP under  continual observation if the output sent to each user is a function of the user's private data and a common quantity computed using standard differential privacy. Note that at each step $h$ of any episode $k$, in both algorithms, the estimators (e.g., $\hat{\theta}_{h,p}^k$) and bonus terms (e.g., $\Gamma_{h,p}^k$) are computed via the four counters, each with $O(\epsilon/4,\delta/4)$-DP guarantee. By composition and post-processing properties of DP, we can conclude that the sequences of policies $\{\pi^k\}_{k \in [K]}$ are computed using an $O(\epsilon,\delta)$-DP mechanism. Now, the actions $\{a_{h}^k\}_{h\in H}$ during episode $k$ are generated via the policy $\pi^k$ and the user's data $s_{h}^k$ as $a_{h}^k = \pi^k(s_{h}^k)$. Then, by the billboard lemma, the composition of the controls $\{a_h^k\}_{k\in K, h\in H}$ sent to all the users is $O(H\epsilon,H\delta)$-JDP. Hence, replacing $\epsilon$ and $\delta$ by $\epsilon/H$ and $\delta/H$, yields the final result.
\end{proof}

\subsection{Regret Guarantee}
In this section, we will utilize the general regret bounds for \texttt{LinOpt-VI-Reg} and \texttt{LinOpt-PO-Reg}  to derive the private regret bounds. In particular, thanks to the results in Theorems~\ref{thm:VI} and~\ref{thm:PO}, we only need to determine the three constants $\lambda_{\max}, \lambda_{\min}$ and $\nu$ under the {\priv} ({equivalently the {\reg}} in Algorithms~\ref{alg:UCB-VI} and~\ref{alg:UCB-PO}) such that Assumption~\ref{ass:regularizer} holds. To this end, we will use concentration results for random Gaussian matrix (Lemma~\ref{lem:norm_sym_gaussian}) and chi-square random variable (Lemma~\ref{lem:chi-square}) to arrive at the following theorem. 
\begin{theorem}
\label{thm:priv_regret}
Let Assumption~\ref{def:lin} hold. For any privacy parameters $\epsilon > 0$ and $\delta \in (0,1]$, and for any $\alpha \in (0,1]$, with probability at least $1-4\alpha$, \texttt{Private-LinOpt-VI} enjoys the regret bound 
\begin{align}
\label{eq:regret_VI}
    \mathcal{R}(T) = O\left( d\sqrt{TH^3} \log(T/\alpha) + \frac{\log(H/\delta)^{1/4}}{\epsilon^{1/2}} d^{7/4}\sqrt{TH^4} \log T \log(T/\alpha)^{1/4} \right).
\end{align}
Similarly, for any $\alpha \in (0,1]$, with probability at least $1-4\alpha$, \texttt{Private-LinOpt-PO} enjoys the regret bound 
\begin{align}
    \mathcal{R}(T) = O\left( d\sqrt{TH^3} \log(T/\alpha) + \sqrt{TH^3 \log |\mathcal{A}|} + \frac{\log(H/\delta)^{1/4}}{\epsilon^{1/2}} d^{7/4}\sqrt{TH^4} \log T \log(T/\alpha)^{1/4} \right).
\end{align}
Moreover, if $\sqrt{\log |\mathcal{A}|} = O(\log \frac{H}{\alpha} + d\log (1+KH^2))$, then \texttt{Private-LinOpt-PO} has the same regret bound as given in~\eqref{eq:regret_VI}.
\end{theorem}

\begin{proof}
As mentioned above, we only need to verify that the {\priv} satisfies Assumption~\ref{ass:regularizer} with proper choices of $\lambda_{\max,p}, \lambda_{\max,r}$, $\lambda_{\min,p}, \lambda_{\min,r}$ and $\nu_p, \nu_r$. We will illustrate the idea behind $\lambda_{\max,p}, \lambda_{\min,p}$ and $\nu_{p}$ here and the same approach applies to the other three constants. By property (ii) in Lemma~\ref{lem:fact}, we know that the total noise added into $\Lambda_{h,p}^k$ is a summation of $m$ symmetric random Gaussian matrix. That is, $N_{h,p}^k$ is a symmetric matrix with its $(i,j)$-th entry, $1\le i\le j\le d_1$, being i.i.d $\mathcal{N}(0, m\sigma_{p,1}^2)$. Therefore, by a tighter concentration bound for symmetric random Gaussian matrix (Lemma~\ref{lem:norm_sym_gaussian} in the appendix), we have with probability at least $1-\alpha/(4KH)$, for all $k\in [K], h\in[H]$
\begin{align*}
    \norm{N_{h,p}^k}\le \Sigma_{p,1} := \sigma_{p,1}\sqrt{m}\left(4\sqrt{d} + \sqrt{8\log(8KH/\alpha)}\right).
\end{align*}
Recall that $Z_{h,p}^k$ is a shift version of $N_{h,p}^k$ to guarantee that it is positive definite. In particular, we choose  $Z_{h,p}^k = N_{h,p}^k + 2\Sigma_{p,1} I$, which leads to $\lambda_{\max,p} = 3\Sigma_{p,1}$ and  $\lambda_{\min,p} = \Sigma_{p,1}$.

Similarly, the total noise in $u_{h,p}^k$ is a summation of $m$ random Gaussian vectors. That is, $z_{h,p}^k$ is a $d_1 \times 1$ vector whose each entry is i.i.d $\mathcal{N}(0, m\sigma_{p,2}^2)$. As a result, $\norm{z_{h,p}^k}^2/(m\sigma_{p,2}^2)$ a $\chi^2$-statistic with $d_1$ degrees of freedom, and hence by concentration bound of chi-square random variable (Lemma~\ref{lem:chi-square} in the appendix), we have with probability at least $1-\alpha/(4KH)$, for all $k\in [K], h\in[H]$
\begin{align*}
    \norm{z_{h,p}^k} \le \sqrt{m}\sigma_{p,2}\left(\sqrt{d_1} + \sqrt{2\log(4KH/\alpha)}\right).
\end{align*}
Thus, via a union bound, we have with probability at least $1-\alpha/2$, for all $h \in [H]$ and $k \in [K]$, 
\begin{align*}
    \norm{Z_{h,p}^k } \leq \lambda_{\max,p}, \;\; \norm{(Z_{h,p}^k)^{-1}} \leq 1/\lambda_{\min,p} \;\; \text{and}\;\; \norm{z_{h,p}^k}_{(Z_{h,p}^k)^{-1}} \leq \nu_{p},
\end{align*}
where $\nu_p = \sigma_{p,2}\sqrt{m/\Sigma_{p,1}}\left(\sqrt{d_1} + \sqrt{2\log(4KH/\alpha)}\right)$. Moreover, via the same analysis, we also have with probability at least $1-\alpha/2$, for all $h \in [H]$ and $k \in [K]$, 
\begin{align*}
    \norm{Z_{h,r}^k } \leq \lambda_{\max,r}, \;\; \norm{(Z_{h,r}^k)^{-1}} \leq 1/\lambda_{\min,r} \;\; \text{and}\;\; \norm{z_{h,r}^k}_{(Z_{h,r}^k)^{-1}} \leq \nu_{r},
\end{align*}
where $\lambda_{\max,r} = 3\lambda_{\min, r} = 3 \Sigma_r$,   $\nu_p = \sigma_{r,2}\sqrt{m/\Sigma_r}\left(\sqrt{d_2} + \sqrt{2\log(4KH/\alpha)}\right)$ and 
\begin{align*}
    \Sigma_r = \sigma_{r,1}\sqrt{m}\left(4\sqrt{d_2} + \sqrt{8\log(8KH/\alpha)}\right).
\end{align*}
Recall that to guarantee $O(\epsilon,\delta)$-JDP in Theorem~\ref{thm:priv}, we choose $\sigma_{p,1} = \frac{ H\Delta_{p,1}}{\epsilon} \sqrt{32m \ln(4H/\delta)}, \sigma_{p,2} = \frac{ H\Delta_{p,2}}{\epsilon} \sqrt{32m \ln(4H/\delta)}, \sigma_{r,1} = \frac{ H\Delta_{r,1}}{\epsilon} \sqrt{32m \ln(4H/\delta)}, \sigma_{r,2} = \frac{ H\Delta_{r,2}}{\epsilon} \sqrt{32m \ln(4H/\delta)}$ with $\Delta_{p,1} = d_1H^2 \le dH^2$, $\Delta_{p,2} = \sqrt{d_1} H^2 \le \sqrt{d}H^2$, $\Delta_{r,1} = \Delta_{r,2} = 1$. Substituting these constants into the regret bound in Theorems~\ref{thm:VI} and~\ref{thm:PO}, yields the final results. 
\end{proof}

\begin{remark}[Bound on $d$]
\label{rem:d}
Note that $d^{7/4}$ in the second term (i.e., private term) of both regret bounds comes from simply bounding both $d_1$ and $d_2$ by $d$. A tighter bound on $d$ is straightforward if it is written in an explicit form of $d_1$ and $d_2$. Note that if we assume that $\norm{\phi(s,a)} \le H$ (as in~\citep{he2021logarithmic,zhou2021nearly} ) rather than $\sqrt{d}H$ in Assumption~\ref{def:lin}, the private dependency on $d$ then reduces from $d^{7/4}$ to $d^{5/4}$.
\end{remark}

\begin{remark}[Comparisons with non-private algorithms]
{The non-private value-based algorithm UCRL-VTR~\citep{jia2020model} enjoys a regret upper bound $\widetilde{O}(d\sqrt{H^3T})$ and the non-private policy-based algorithm OPPO~\citep{cai2020provably} has a regret upper bound $\widetilde{O}(d\sqrt{TH^3} + \sqrt{TH^3 \log|\mathcal{A}|})$. On the other hand, \texttt{Private-LinOpt-VI} has a regret bound $\widetilde{O}(d^{7/4}\sqrt{TH^4} (\log(H/\delta)^{1/4}) \sqrt{1/\epsilon})$ and \texttt{Private-LinOpt-PO} has a regret bound $\widetilde{O}(d^{7/4}\sqrt{TH^4} (\log(H/\delta)^{1/4}) \sqrt{1/\epsilon} +\sqrt{TH^3 \log|\mathcal{A}|})$}. 
\end{remark}

\begin{remark}[Comparisons tabular private RL result]
{We compare our results to~\citep{vietri2020private}, which establishes regret for tabular value-based RL under $(\epsilon,0)$-JDP guarantee by injecting Laplace noise via binary counting mechanism. In particular, the claimed regret is $\widetilde{O}\left(\sqrt{SAH^3T} +  S^2AH^3/\epsilon\right)$, i.e., both non-private and private terms are polynomial in terms of the number of states and actions while ours are independent of them. Moreover, the dependence of $\epsilon$ is an additive term of $1/\epsilon$ in~\citep{vietri2020private} while ours have a multiplicative factor $1/\sqrt{\epsilon}$ under $(\epsilon,\delta)$-JDP. Finally, since tabular MDP is a special case of linear mixture MDP, our established regret bounds directly imply bounds for the tabular case (although may not be optimal in general through this reduction.)}

\end{remark}

\begin{remark}[General private RL design]
\label{rem:ldp}
As mentioned before, our two proposed algorithms (\texttt{LinOpt-VI-Reg} and \texttt{LinOpt-PO-Reg}) along with their regret bounds (Theorems~\ref{thm:VI} and~\ref{thm:PO}) provide a unified procedure to designing various private RL algorithms by varying the choice of {\reg} (or equivalently {\priv}). This not only enables to design new private RL algorithms with JDP guarantees, but allows us to establish RL algorithms with LDP guarantees. In this case, instead of a central {\priv} at the agent,  a local {\priv} at each user's side is used to inject independent Gaussian noise to user's trajectories before being collected by the agent. However, our analysis smoothly carries out and the only change is to determine the three new constants $\lambda_{\max}, \lambda_{\min}, \nu$ under stronger noise in the LDP setting, which can be obtained via the same approach in our current analysis. 

\end{remark}

\subsection{Discussions}
{In this paper, for the purpose of theoretical analysis, we assume realizability, i.e. the MDP in fact lies within the class of linear mixture MDPs. One important future direction is to consider misspecified linear models (cf.~\citep{jin2020provably, lattimore2020learning}).}
Another important direction is to  apply our flexible framework to MDPs with other types of function approximation, e.g., kernelized MDPs or  overparameterized neural network \citep{yang2020function} and generic MDPs \citep{ayoub2020model}. Specifically, at least in the kernelized MDPs settings (i.e.,  reproducing kernel
Hilbert space (RKHS)), it turns out that our current analysis can be easily carried out by using the standard  Quadrature Fourier Features (QFF) approximations~\citep{mutny2019efficient}. It is worth noting that one drawback of QFF approximation is that it only works with standard squared exponential kernel but not with more practical Mat\'ern kernel~\citep{calandriello2019gaussian}.
One interesting question is whether it is possible to use adaptive approximation methods (e.g., Nystr\"{o}m Embedding~\citep{calandriello2019gaussian}) that works for both kernels in the private RL learning. The difficulty here is that the truncated dimension is changing in each round.
In the setting with general function approximations, it is still not clear how to generalize our current analysis with a computationally efficient private scheme such as the binary counting mechanism used in our paper.
Finally, it will also be interesting to consider other privacy guarantees beyond JDP and LDP, e.g., \emph{shuffle model of DP}~\citep{cheu2019distributed}, that achieves a smooth transition between JDP and LDP. This is particularly useful even in the simpler linear bandit setting as the regret under LDP is $O(T^{3/4})$~\citep{zheng2020locally} while the regret under JDP is $O(\sqrt{T})$~\citep{shariff2018differentially}. Thus, one interesting question is whether the shuffle model of DP can be utilized to achieve a similar regret as in JDP while providing the same strong privacy guarantee as in LDP, which is one of ongoing works.

\section{Proofs of Theorems~\ref{thm:VI} and~\ref{thm:PO}}
In this section, we present the detailed proofs for Theorems~\ref{thm:VI} and~\ref{thm:PO}. To start with, we define the following filtration which will be frequently used in the proofs. This filtration can be regarded as the bandit-variant of the one in~\citep{cai2020provably}.

\begin{definition}[Filtration]
\label{def:fil}
For any $(k,h) \in [K] \times [H]$, we define $\mathcal{F}_{k,h,1}$ as the $\sigma$-algebra generated by the following state-action-reward sequence,
\begin{align*}
    \{(s_i^{\tau}, a_i^{\tau}, r_i^{\tau}) \}_{(\tau, i)\in [k-1] \times [H]} \cup \{(s_i^k, a_i^k, r_i^k)_{i \in [h]}\}
\end{align*}
and $\mathcal{F}_{k,h,2}$ as the $\sigma$-algebra generated by 
\begin{align*}
    \{(s_i^{\tau}, a_i^{\tau}, r_i^{\tau}) \}_{(\tau, i)\in [k-1] \times [H]} \cup \{(s_i^k, a_i^k, r_i^k)_{i \in [h]}\} \cup \{s_{h+1}^k\},
\end{align*}
where we define $s_{H+1}^k$ as a null state for any $k \in [K]$. The $\sigma$-algebra sequence $\{\mathcal{F}_{k,h,m}\}_{(k,h,m) \in [K]\times [H]\times 2}$ is a filtration with respect to the timestep index 
\begin{align*}
    t(k,h,m) = 2H(k-1) + 2(h-1) + m.
\end{align*}
In other words, for any $t(k,h,m) \le t(k',h',m')$, it holds that $\mathcal{F}_{k,h,m} \subseteq \mathcal{F}_{k',h',m'}$.
\end{definition}

By the definition, we immediately see that the estimated value function $V_h^k$ and $Q$-function $Q_h^k$ in both algorithms are measurable to $\mathcal{F}_{k,1,1}$ as they are obtained based on the $(k-1)$ historical trajectories. 
Moreover, throughout the proofs, we let $\mathcal{E}(\alpha): = \{\theta_{h,p} \in \mathcal{C}_{h,p}^{k} \text{ and }  \theta_{h,r} \in \mathcal{C}_{h,r}^{k}, \forall h \in [H], k \in [K]\}$ denote the `good' event, which by Lemma~\ref{lem:CI}, holds with probability at least $1-3\alpha$.

\subsection{Proof of Theorem~\ref{thm:VI}}
\label{sec:proofthmvi}
We first show that the estimated $Q$-function and value function in Algorithm~\ref{alg:UCB-VI} is optimistic. First note that, under the nice event given by $\mathcal{E}(\alpha)$, we have that for all $s,a,k,h$
\begin{align*}
    &\vp(s,a)^{\top} \hat{\bt}_{h,r}^k+ \phi_h^k(s,a)^{\top} \hat{\bt}_{h,p}^k + ( {\Gamma}_{h,p}^k +  {\Gamma}_{h,r}^k)(s,a) \ge r_h(s,a) + [\mathbb{P}_h V_{h+1}^k](s,a) \ge 0
\end{align*}
Thus, under event $\mathcal{E}(\alpha)$, we can remove the truncation at zero in the update of $Q$ function in Algorithm~\ref{alg:UCB-VI}. 
\begin{lemma}
\label{lem:op}
Under event $\mathcal{E}(\alpha)$, for any $s,a,k,h$, we have $Q_{h}^k(s,a) \ge Q^*_{h}(s,a)$ and $V_{h}^k(s,a) \ge V^*_{h}(s,a)$.
\end{lemma}
\begin{proof}
The proof is standard (cf.~\citep[Lemma C.4]{ding2021provably}). We present the proof here for ease of reference. The result is proved by induction. First note that the base case for $h=H+1$ holds trivially since all the values are zero. Fix a $k$, suppose the result holds $h+1$, i.e., $Q_{h}^{k+1}(\cdot,\cdot) \ge Q^*_{h+1}(\cdot,\cdot)$, $V_{h+1}^k(\cdot,\cdot) \ge V^*_{h+1}(\cdot,\cdot)$. Given $s,a$, if $Q_{h}^k(s,a) \ge H$, then $Q_{h}^k(s,a) \ge H \ge Q^*_{h}(s,a)$. Otherwise, we have 
\begin{align*}
    &Q_{h}^k(s,a) - Q^*_{h}(s,a)\\
    \ep{a}&\inner{\hat{\bt}_{h,p}^k - \bt_{h,p}}{\phi_h^k(s_h^k,a_h^k)} \!+\! \inner{\hat{\bt}_{h,r}^k - \bt_{h,r}}{\vp(s_h^k, a_h^k)} 
    \!+\! [\mathbb{P}_hV_{h+1}^k](s_h^k,a_h^k)\nonumber \\
    &\!+\! \beta_p(\alpha)\norm{\phi_h^k(s_h^k,a_h^k)}_{(\Lambda_{h,p}^k)^{-1}} +\beta_r(\alpha)\norm{\vp(s_h^k,a_h^k)}_{(\Lambda_{h,r}^k)^{-1}} \!-\![\mathbb{P}_hV_{h+1}^{*}](s_h^k,a_h^k)\nonumber\\
    \gep{b}& [\mathbb{P}_hV_{h+1}^k](s,a) - [\mathbb{P}_hV_{h+1}^{*}](s,a) \\
    \gep{c}&  0
\end{align*}
where (a) holds by adding and subtracting the same term, which holds by the linear kernel MDP model; (b) holds by event $\mathcal{E}(\alpha)$ and Cauchy-Schwarz inequality; (c) holds by the induction assumption and the fact that $\mathbb{P}_h$ is a monotone operator with respect the partial ordering of functions. Thus, for all $(s,a)$, we have established that $Q_{h}^k(s,a) \ge Q^*_{h}(s,a)$, which directly implies that $V_{h}^k(s,a) \ge V^*_{h}(s,a)$.
\end{proof}

Now, we are ready to present the proof of Theorem~\ref{thm:VI}.
\begin{proof}
First, we have under event $\mathcal{E}(\alpha)$,
\begin{align}
    &V_{h}^k(s_h^k) - V_{h}^{\pi^k}(s_h^k)\nonumber\\
    \lep{a}& \inner{\hat{\bt}_{h,p}^k}{\phi_h^k(s_h^k,a_h^k)} +\inner{\hat{\bt}_{h,r}^k}{\vp(s_h^k,a_h^k)} + \beta_p(\alpha)\norm{\phi_h^k(s_h^k,a_h^k)}_{(\Lambda_{h,p}^k)^{-1}} +\beta_r(\alpha)\norm{\vp(s_h^k,a_h^k)}_{(\Lambda_{h,r}^k)^{-1}}\nonumber \\
    & r_h(s_h^k, a_h^k)  - [\mathbb{P}_hV_{h+1}^{\pi^k}](s_h^k,a_h^k)\nonumber\\
    \ep{b}& \inner{\hat{\bt}_{h,p}^k - \bt_{h,p}}{\phi_h^k(s_h^k,a_h^k)} \!+\! \inner{\hat{\bt}_{h,r}^k - \bt_{h,r}}{\vp(s_h^k, a_h^k)} 
    \!+\! [\mathbb{P}_hV_{h+1}^k](s_h^k,a_h^k)\nonumber \\
    &\!+\! \beta_p(\alpha)\norm{\phi_h^k(s_h^k,a_h^k)}_{(\Lambda_{h,p}^k)^{-1}} +\beta_r(\alpha)\norm{\vp(s_h^k,a_h^k)}_{(\Lambda_{h,r}^k)^{-1}} \!-\![\mathbb{P}_hV_{h+1}^{\pi^k}](s_h^k,a_h^k)\nonumber\\
    \lep{c}&[\mathbb{P}_hV_{h+1}^k](s_h^k,a_h^k) -[\mathbb{P}_hV_{h+1}^{\pi^k}](s_h^k,a_h^k) + 2\beta_{p}(\alpha)\norm{\phi_h^k(s_h^k,a_h^k)}_{(\Lambda_{h,p}^k)^{-1}}+2\beta_r(\alpha)\norm{\vp(s_h^k,a_h^k)}_{(\Lambda_{h,r}^k)^{-1}}\label{eq:first}
\end{align}
where (a) holds by the update of $V_{h}^k$ in Algorithm~\ref{alg:UCB-VI}, the Bellman equation for $V_h^{\pi^k}$, and  the greedy selection of the algorithm; (b) holds by the linear mixture MDP model; (c) holds by Cauchy-Schwarz inequality and the event of $\mathcal{E}(\alpha)$.

Now, by using Lemma~\ref{lem:op}, we can further bound~\eqref{eq:first} as follows. 
\begin{align*}
     &V_{h}^k(s_h^k) - V_{h}^{\pi^k}(s_h^k)\\
     \lep{a}&\min\left\{H,[\mathbb{P}_hV_{h+1}^k](s_h^k,a_h^k) -[\mathbb{P}_hV_{h+1}^{\pi^k}](s_h^k,a_h^k) + 2\beta_{p}(\alpha)\norm{\phi_h^k(s_h^k,a_h^k)}_{(\Lambda_{h,p}^k)^{-1}}+2\beta_r(\alpha)\norm{\vp(s_h^k,a_h^k)}_{(\Lambda_{h,r}^k)^{-1}}\right\}\\
     \lep{b}& \min\left\{H, 2\beta_{p}(\alpha)\norm{\phi_h^k(s_h^k,a_h^k)}_{(\Lambda_{h,p}^k)^{-1}}+2\beta_r(\alpha)\norm{\vp(s_h^k,a_h^k)}_{(\Lambda_{h,r}^k)^{-1}}\right\} + [\mathbb{P}_hV_{h+1}^k](s_h^k,a_h^k) -[\mathbb{P}_hV_{h+1}^{\pi^k}](s_h^k,a_h^k)\\
     \lep{c}&2\beta_p(\alpha)\min\left\{1, \norm{\phi_h^k(s_h^k,a_h^k)}_{(\Lambda_{h,p}^k)^{-1}}\right\} + 2H\beta_r(\alpha)\min\left\{1, \norm{\vp(s_h^k,a_h^k)}_{(\Lambda_{h,r}^k)^{-1}}\right\} \\
     &+ [\mathbb{P}_hV_{h+1}^k](s_h^k,a_h^k) -[\mathbb{P}_hV_{h+1}^{\pi^k}](s_h^k,a_h^k)
\end{align*}
where (a) holds by the fact that value function is bounded in the range $[0,H]$; (b) is true since $V_{h+1}^k(\cdot) \ge V_{h+1}^{\pi^k}(\cdot)$ by Lemma~\ref{lem:op} on event $\mathcal{E}(\alpha)$; (c) is truce since $2\beta_p(\alpha) \ge H$ and $2H\beta_r(\alpha) \ge H$  by the definitions of $\beta_p(\alpha)$ and $\beta_r(\alpha) $ in Lemma~\ref{lem:CI}. Therefore, we have 
\begin{align*}
    &V_{h}^k(s_h^k) - V_{h}^{\pi^k}(s_h^k) - \left(V_{h+1}^k(s_{h+1}^k) - V_{h+1}^{\pi^k}(s_{h+1}^k)\right)\\
    \le&2\beta_p(\alpha)\min\left\{1, \norm{\phi_h^k(s_h^k,a_h^k)}_{(\Lambda_{h,p}^k)^{-1}}\right\} + 2H\beta_r(\alpha)\min\left\{1, \norm{\vp(s_h^k,a_h^k)}_{(\Lambda_{h,r}^k)^{-1}}\right\} \\
    &+ \mathbb{P}_h[V_{h+1}^k - V_h^{\pi^k}](s_h^k,a_h^k) - [V_{h+1}^k - V_h^{\pi^k}](s_{h+1}^k).
\end{align*}
Therefore, taking a summation over $k\in [K]$ and $h \in [H]$, yields that under event  $\mathcal{E}(\alpha)$
\begin{align*}
    \sum_{k=1}^K V_{1}^k(s_1^k) - V_{1}^{\pi^k}(s_1^k) &\le \underbrace{\sum_{k=1}^K \sum_{h=1}^H2\beta_p(\alpha)\min\left\{1, \norm{\phi_h^k(s_h^k,a_h^k)}_{(\Lambda_{h,p}^k)^{-1}}\right\} }_{\mathcal{T}_1}\\
    &+ \underbrace{\sum_{k=1}^K \sum_{h=1}^H2H\beta_r(\alpha)\min\left\{1, \norm{\vp(s_h^k,a_h^k)}_{(\Lambda_{h,r}^k)^{-1}}\right\}}_{\mathcal{T}_2}\\
    &+ \underbrace{\sum_{k=1}^K\sum_{h=1}^H \mathbb{P}_h[V_{h+1}^k - V_h^{\pi^k}](s_h^k,a_h^k) - [V_{h+1}^k - V_h^{\pi^k}](s_{h+1}^k)}_{\mathcal{T}_3}.
\end{align*}
Now, we are left to bound each of the two terms. First, $\mathcal{T}_3$ is a summation of a bounded martingale difference sequence, with boundedness being $H$. To see this, we let $Y_{k,h}:=\mathbb{P}_h[V_{h+1}^k - V_h^{\pi^k}](s_h^k,a_h^k) - [V_{h+1}^k - V_h^{\pi^k}](s_{h+1}^k)$. Then, we have $Y_{k,h} \in \mathcal{F}_{k,h,2}$ and $\ex{ Y_{k,h} \mid \mathcal{F}_{k,h,1}} = 0$. Thus, by the Azuma–Hoeffding inequality, we have that for any fixed $\alpha \in (0,1]$,  with probability at least $1-\alpha$,
\begin{align*}
    \mathcal{T}_3 \le H\sqrt{2KH\log(1/\alpha)}.
\end{align*}
To bound $\mathcal{T}_1$ and $\mathcal{T}_2$, we will use the standard elliptical potential lemma (cf.~\citep[Lemma 11]{abbasi2011improved}). In particular, we have 
\begin{align*}
    \mathcal{T}_1 &\lep{a} 2\beta_p(\alpha)\sqrt{KH}\sqrt{\sum_{k=1}^K\sum_{h=1}^H \min\left\{1, \norm{\phi_h^k(s_h^k,a_h^k)}_{(\Lambda_{h,p}^k)^{-1}}^2\right\}}\\
    &\lep{b}2\beta_p(\alpha)\sqrt{KH}\sqrt{2Hd_1\log(1+KH^2/(\lambda_{\min,p}))},
\end{align*}
where (a) holds by Cauchy-Schwarz  inequality; (b) follows from Lemma 11 in~\citep{abbasi2011improved} and the last step in the proof of Lemma~\ref{lem:CI}. Similarly, we have 
\begin{align*}
    \mathcal{T}_2 &\le 2H\beta_r(\alpha)\sqrt{KH}\sqrt{\sum_{k=1}^K\sum_{h=1}^H \min\left\{1, \norm{\vp(s_h^k,a_h^k)}_{(\Lambda_{h,r}^k)^{-1}}^2\right\}}\\
    &\le 2H\beta_r(\alpha)\sqrt{KH}\sqrt{2Hd_2\log(1+K/(d_2\lambda_{\min,r}))}.
\end{align*}
Finally, note that $V_1^*(s_1^k) \le V_{1}^k(s_1^k)$ for all $k$ by Lemma~\ref{lem:op}. Putting everything together, we have for any $\alpha \in (0,1]$, with probability at least $1-4\alpha$, Algorithm~\ref{alg:UCB-VI} achieves a regret bound
\begin{align*}
    \mathcal{R}(T) &= O\left(\sqrt{TH^3}\left(\log \frac{H}{\alpha} + d \kappa_{\lambda_{\min}}\right) + \sqrt{TH d \kappa_{\lambda_{\min}} } \left(\sqrt{d \lambda_{\max}} + \nu\right) \right)
\end{align*}
where $\kappa_{\lambda_{\min}}:=\log\left(1+KH^2/(\lambda_{\min})\right)$ and $\lambda_{\min} = \min\{\lambda_{\min, p}, \lambda_{\min, r}\}$, $\lambda_{\max} = \max\{\lambda_{\max, p}, \lambda_{\max, r}\}$ and $\nu = \max\{\nu_{p}, \nu_{r}\} $
\end{proof}

\section{Proof of Theorem~\ref{thm:PO}}
\label{sec:proofthmpo}
\begin{proof}
We start with the following standard regret decomposition (cf.~\citep[Lemma 4.2]{cai2020provably}). 
\begin{align*}
\mathcal{R}(T)  &= \sum_{k=1}^K\left(V_1^*(s_1^k) - V_1^{\pi_k} (s_1^k)\right)\\
&= \underbrace{\sum_{k=1}^K \sum_{h=1}^H \mathbb{E}_{\pi^*}\left[ \inner{Q_h^k(s_h^k)}{\pi_h^*(\cdot | s_h)- \pi_h^{\pi_k}(\cdot | s_h)} \mid s_1 = s_1^k\right] }_{(i)} + \underbrace{\mathcal{M}_{K,H,2}}_{(ii)}\\
&+ \underbrace{\sum_{k=1}^K \sum_{h=1}^H \left(\mathbb{E}_{\pi^*}[\iota_h^k(x_h,a_h) \mid s_1 = s_1^k] - \iota_h^k(s_h^k, a_h^k)\right) }_{(iii)},
\end{align*}
where $\iota_h^k:=r_h + \mathbb{P}_h V_{h+1}^k - Q_h^k$ and $\mathbb{E}_{\pi}[\cdot]$ is taking expectation with respect to $\pi$ with transitions  $\mathbb{P}_h$.

Here, $\{\mathcal{M}_{k,h,m}\}_{(k,h,m) \in [K]\times[H]\times [2]}$ is a martingale adapted to the filtration $\{\mathcal{F}_{k,h,,}\}_{(k,h,m) \in [K]\times[H]\times [2]}$ with respect to the timestep in Definition~\ref{def:fil}. Specifically, we have 
\begin{align*}
    \mathcal{M}_{K,H,2} = \sum_{k=1}^K\sum_{h=1}^H (D_{k,h,1} + D_{k,h,2}),
\end{align*}
where 
\begin{align*}
    &D_{k,h,1} = \left(\mathbb{J}_{k,h}(Q_h^k - Q_h^{\pi^k}) \right)(s_h^k) - (Q_h^k - Q_h^{\pi^k})(s_h^k, a_h^k)\\
    &D_{k,h,2} = \mathbb{P}_h[V_{h+1}^k - V_h^{\pi^k}](s_h^k,a_h^k) - [V_{h+1}^k - V_h^{\pi^k}](s_{h+1}^k),
\end{align*}
and $(\mathbb{J}_{k,h}f)(x) :=\inner{f(x,\cdot)}{\pi_h^k(\cdot\mid x)}$. Thus, by the boundedness of value functions and $Q$-functions along with Azuma–Hoeffding inequality, we have $ (ii) \le 4H\sqrt{2T \log(1/\alpha) }$. The first term can be bounded by using standard property of mirror-descent update and the fact that $\pi^1$ is a uniform distribution. In particular, as in~\citep[Eq. C.2]{cai2020provably}, we have $  (i) \le \sqrt{2H^3T\log |\mathcal{A}|}$ when $\eta = \sqrt{2\log |\mathcal{A}|/(HT)}$ in Algorithm~\ref{alg:UCB-PO}. The key part is to bound the third term. To this end, we will show that 
under good event $\mathcal{E}(\alpha)$, we have for all $s,a,k,h$, 
\begin{align}
\label{eq:opt}
    2 ( {\Gamma}_{h,p}^k +  {\Gamma}_{h,r}^k)(s,a) \ge -\iota_h^k(s,a) \ge 0.
\end{align}
To show this result, we first note that under $\mathcal{E}(\alpha)$, for all $s,a,k,h$
\begin{align*}
    &\vp(s,a)^{\top} \hat{\bt}_{h,r}^k+ \phi_h^k(s,a)^{\top} \hat{\bt}_{h,p}^k + ( {\Gamma}_{h,p}^k +  {\Gamma}_{h,r}^k)(s,a) \ge r_h(s,a) + [\mathbb{P}_h V_{h+1}^k](s,a) \ge 0
\end{align*}
Thus, under event $\mathcal{E}(\alpha)$, we can remove the truncation at zero in the update of $Q$ function in Algorithm~\ref{alg:UCB-PO}. Thus, we have 
\begin{align*}
    -\iota_h^k(s,a) &= Q_h^k(s,a) - (r_h(s,a) + [\mathbb{P}_h V_{h+1}^k](s,a))\\
    &\le \vp(s,a)^{\top} \hat{\bt}_{h,r}^k+ \phi_h^k(s,a)^{\top} \hat{\bt}_{h,p}^k + ( {\Gamma}_{h,p}^k +  {\Gamma}_{h,r}^k)(s,a)-(r_h(s,a) + [\mathbb{P}_h V_{h+1}^k](s,a))\\
    &\le  2 ( {\Gamma}_{h,p}^k +  {\Gamma}_{h,r}^k)(s,a).
\end{align*}

Meanwhile, we have under event $\mathcal{E}(\alpha)$
\begin{align*}
    &\iota_h^k(s,a) \\
    =& (r_h(s,a) + [\mathbb{P}_h V_{h+1}^k](s,a)) - Q_h^k(s,a)\\
    =& (r_h(s,a) + [\mathbb{P}_h V_{h+1}^k](s,a)) - \min\{H,\vp(s,a)^{\top} \hat{\bt}_{h,r}^k+ \phi_h^k(s,a)^{\top} \hat{\bt}_{h,p}^k + ( {\Gamma}_{h,p}^k +  {\Gamma}_{h,r}^k)(s,a)\}\\
    =&\max\{r_h(s,a) + [\mathbb{P}_h V_{h+1}^k](s,a)-H, r_h(s,a) + [\mathbb{P}_h V_{h+1}^k](s,a) - \vp(s,a)^{\top} \hat{\bt}_{h,r}^k- \phi_h^k(s,a)^{\top} \hat{\bt}_{h,p}^k - ( {\Gamma}_{h,p}^k +  {\Gamma}_{h,r}^k)(s,a) \}\\
    \le& 0
\end{align*}

Therefore, similar to the proof of Theorem~\ref{thm:VI}, we have 
\begin{align*}
    (iii) = O\left(\sqrt{TH^3}\left(\log \frac{H}{\alpha} + d \kappa_{\lambda_{\min}}\right) + \sqrt{TH d \kappa_{\lambda_{\min}} } \left(\sqrt{d \lambda_{\max}} + \nu\right) \right)
\end{align*}
where $\kappa_{\lambda_{\min}}:=\log\left(1+KH^2/(\lambda_{\min})\right)$ and $\lambda_{\min} = \min\{\lambda_{\min, p}, \lambda_{\min, r}\}$, $\lambda_{\max} = \max\{\lambda_{\max, p}, \lambda_{\max, r}\}$ and $\nu = \max\{\nu_{p}, \nu_{r}\} $.

Putting everything together, we have 
\begin{align*}
    \mathcal{R}(T) = O\left(\sqrt{TH^3}\left(\log \frac{H}{\alpha} + d \kappa_{\lambda_{\min}}\right) +  \sqrt{TH^3 \log |\mathcal{A}|} + \sqrt{TH d \kappa_{\lambda_{\min}} } \left(\sqrt{d \lambda_{\max}} + \nu\right) \right)
\end{align*}
\end{proof}

\section{Conclusion}
In this work, we took first step towards private RL algorithms with linear function approximations. We proposed  both value-based and policy-based optimistic private RL algorithms and show that they enjoy sublinear regret in total number of steps while guaranteeing joint differential privacy (JDP). Moreover, these regret bounds are independent of the number of states and at most depends logarithmically on the number of actions in the policy optimization setting. Our established results build on a flexible procedure that also allows us to design general private RL algorithms with different privacy schemes under JDP and even under LDP guarantees by injecting Gaussian noise directly to each user's data. As discussed before, our work also opens the door to a series of interesting future works.

\section{Acknowledgement}
The author would like to thank Sayak Ray Chowdhury for insightful discussions and Quanquan Gu for the discussion on the concurrent work.

\bibliographystyle{unsrtnat}
\bibliography{main,isit}

\newpage

\begin{appendix}
\section{Proof of Lemma~\ref{lem:CI}}
\label{sec:proofCI}
\begin{proof}
Fix $h \in [H]$. Define $X_{h,p}^k \in \Real^{(k-1)\times d_1}$ to be a matrix whose $j$-th row is $(\phi_h^j(s_h^j,a_h^j))^{\top}$, and a `noise' column vector $\eta_{h,p}^k$ whose $j$-th row is $V_{h+1}^j(s_{h+1}^j) - \inner{\phi_h^j(s_h^j,a_h^j) }{\bt_{h,p}} $. Then, we have
 for any $k\ge 1$,
\begin{align*}
    &\bt_{h,p} - \hat{\bt}_{h,p}^k\\
    =&\bt_{h,p}- [\Lambda_{h,p}^k]^{-1}u_{h,p}^k\\
    =&\bt_{h,p} - [\Lambda_{h,p}^k]^{-1}\left( (X_{h,p}^k) ^{\top} X_{h,p}^k \bt_{h,p} + (X_{h,p}^k)^{\top}\eta_{h,p}^k + z_{h,p}^k \right)\\
    =&\bt_{h,p} - [\Lambda_{h,p}^k]^{-1}\left( \Lambda_{h,p}^k\bt_{h,p} - Z_{h,p}^k\bt_{h,p} + (X_{h,p}^k)^{\top}\eta_{h,p}^k + z_{h,p}^k \right)\\
    =& [\Lambda_{h,p}^k]^{-1}\left(Z_{h,p}^k\bt_{h,p} - (X_{h,p}^k)^{\top}\eta_{h,p}^k - z_{h,p}^k\right).
\end{align*}
Therefore, we have 
\begin{align*}
    \norm{\bt_{h,p}- \hat{\bt}_{h,p}^k}_{\Lambda_{h,p}^k} &= \norm{Z_{h,p}^k\bt_{h,p} - (X_{h,p}^k)^{\top}\eta_{h,p}^k - z_{h,p}^k}_{[\Lambda_{h,p}^k]^{-1}}\\
    &\lep{a}  \norm{(X_{h,p}^k)^{\top}\eta_{h,p}^k}_{[\Lambda_{h,p}^k]^{-1}} + \norm{Z_{h,p}^k\bt_{h,p}}_{[\Lambda_{h,p}^k]^{-1}}+ \norm{z_{h,p}^k}_{[\Lambda_{h,p}^k]^{-1}}\\
    &\lep{b}\norm{(X_{h,p}^k)^{\top}\eta_{h,p}^k}_{[\Lambda_{h,p}^k]^{-1}} + \norm{Z_{h,p}^k\bt_{h,p}}_{[Z_{h,p}^k]^{-1}}+ \norm{z_{h,p}^k}_{[Z_{h,p}^k]^{-1}}\\
    &= \underbrace{\norm{(X_{h,p}^k)^{\top}\eta_{h,p}^k}_{[\Lambda_{h,p}^k]^{-1}}}_{\mathcal{B}_{1,p}} + \underbrace{\norm{\bt_{h,p}}_{Z_{h,p}^k} + \norm{z_{h,p}^k}_{[Z_{h,p}^k]^{-1}}}_{\mathcal{B}_{2,p}},
\end{align*}
where (a) follows from triangle inequality;  (b) holds by ${\Lambda}_{h,p}^k \succeq Z_{h,p}^k$.

By the same argument, we have 
\begin{align*}
    \norm{\bt_{h,r}- \hat{\bt}_{h,r}^k}_{\Lambda_{h,r}^k} \le \underbrace{\norm{(X_{h,r}^k)^{\top}\eta_{h,r}^k}_{[\Lambda_{h,r}^k]^{-1}}}_{\mathcal{B}_{1,r}} + \underbrace{\norm{\bt_{h,r}}_{Z_{h,r}^k} + \norm{z_{h,r}^k}_{[Z_{h,r}^k]^{-1}}}_{\mathcal{B}_{2,r}},
\end{align*}
where $X_{h,r}^k \in \mathbb{R}^{(k-1) \times d_2}$ is a matrix whose rows are $(\vp(s_h^1,a_h^1))^{\top},  \ldots, (\vp(s_h^{k-1},a_h^{k-1}))^{\top}$ and $\eta_{h,r}^k$ is a `noise' column vector whose $j$-th row is $r_h^j(s_h^j, a_h^j) - r_h(s_h^j, a_h^j)$. 

Now, by Assumption~\ref{ass:regularizer}, we have for any $\alpha \in (0,1]$, with probability at least $1-\alpha$, for all $k \in [K]$ and all $h \in [H]$
\begin{align*}
    &\mathcal{B}_{1,p} \le \norm{(X_{h,p}^k)^{\top}\eta_{h,p}^k}_{[G_{h,p}^k + \lambda_{\min,p}]^{-1}} \quad \text{and} \quad \mathcal{B}_{1,r} \le \norm{(X_{h,r}^k)^{\top}\eta_{h,r}^k}_{[G_{h,r}^k + \lambda_{\min,r}]^{-1}} \\
    &\mathcal{B}_{2,p}\le  \sqrt{d_1 \lambda_{\max,p} }+ \nu_{p} \quad \text{and} \quad \mathcal{B}_{2,r}\le  \sqrt{d_2 \lambda_{\max,r} }+ \nu_{r},
\end{align*}
where $G_{h,p}^k := (X_{h,p}^k)^{\top} (X_{h,p}^k)$ and $G_{h,r}^k := (X_{h,r}^k)^{\top} (X_{h,r}^k)$. Moreover, we also used the assumption in Definition~\ref{def:lin} that $\norm{\theta_{h,p}} \le \sqrt{d_1}$ and $\norm{\theta_{h,r}} \le \sqrt{d_2}$.

We are only left with two noise terms, which can be bounded by the standard self-normalized inequality (cf.~\cite[Theorem 1]{abbasi2011improved}). In particular, first note that $j$-th element of the `noise' vector $\eta_{h,p}^j$ is $V_{h+1}^j(s_{h+1}^j) - \inner{\phi_h^j(s_h^j,a_h^j) }{\bt_{h,p}} = V_{h+1}^j(s_{h+1}^j) - \mathbb{E}_{s'\sim \mathbb{P}(\cdot\mid s_h^j,a_h^j)}V_{h+1}^j(s')$, which holds by the linear kernel MDP model in Assumption~\ref{def:lin}. Note that, conditioning on $\mathcal{F}_{j,h,1}$, $\eta_{h,p}^j$ is a zero-mean random variable. This holds by the fact that conditioning on $\mathcal{F}_{j,h,1}$, the only randomness comes from $s_{h+1}^j$. Moreover, since $V_{h+1}^j \in [0,H]$, conditioning on $\mathcal{F}_{j,h,1}$, $\eta_{h,p}^j$ is an $H/2$-sub-Gaussian random variable. Also, $\eta_{h,p}^j$ is $\mathcal{F}_{j,h,2}$ measurable and $\phi_h^j(s_h^j,a_h^j)$ is $\mathcal{F}_{j,h,1}$ measurable by Definition~\ref{def:fil}. Similarly, we can see that the noise $\eta_{h,r}^j$ is a $1/2$-sub-Gaussian random variable conditioning on $\mathcal{F}_{j,h-1,2}$ as $\pi^j$ is measurable with respect to $\mathcal{F}_{j,h,1}$ and hence conditioning on $\mathcal{F}_{j,h-1,2}$, the only randomness comes from the reward realization $r_h^j(s_h^j,a_h^j) \in [0,1]$. Moreover, $\eta_{h,r}^j$ is $\mathcal{F}_{j,h,1}$ measurable and $\vp(s_h^j,a_h^j)$ is $\mathcal{F}_{j,h-1,2}$ measurable. With such choices of filtrations, one can now apply standard self-normalized inequality to obtain that for any fixed $\alpha \in (0,1]$, with probability at least $1-3\alpha$, for all $k \in [K]$ and $h \in  [H]$
\begin{align*}
    \mathcal{B}_{1,p} \le \frac{H}{2}\sqrt{2\log\left(\frac{H}{\alpha}\right) + \log \frac{\det({G}_{h,p}^k + \lambda_{\min,p}I)}{\det(\lambda_{\min,p} I)}}
\end{align*}
and 
\begin{align*}
    \mathcal{B}_{1,r} \le \frac{1}{2}\sqrt{2\log\left(\frac{H}{\alpha}\right) + \log \frac{\det({G}_{h,r}^k + \lambda_{\min,r}I)}{\det(\lambda_{\min,r} I)}}.
\end{align*}

Then, we can further bound the two terms by using the trace-determinant lemma (cf.~\cite[Lemma 11]{abbasi2011improved}) and the boundedness assumptions in Assumption~\ref{def:lin}, that is, for all all $h \in [H]$ and $k \in [K]$,  $\norm{\phi_h^k(s_h^k, a_h^k) } \le \sqrt{d_1} H$ and $\norm{\vp(s_h^k,a_h^k)} \le 1$. In particular, we obtain that with probability at least $1-3\alpha$, for all $k \in [K]$ and $h \in [H]$
\begin{align*}
    \mathcal{B}_{1,p} \le \frac{H}{2}\sqrt{2\log\left(\frac{H}{\alpha}\right) +  d_1\log\left(1 + \frac{KH^2}{\lambda_{\min,p}}\right)}
\end{align*}
and 
\begin{align*}
    \mathcal{B}_{1,r} \le \frac{1}{2}\sqrt{2\log\left(\frac{H}{\alpha}\right) +  d_2\log\left(1 + \frac{K}{d_2\lambda_{\min,r}}\right) }.
\end{align*}
Finally, putting everything together, we have for any fixed $\alpha \in (0,1]$, with probability at least $1-3\alpha$ for all $k \in [K]$ and $h \in [H]$,
\begin{align*}
     \norm{\bt_{h,p}- \hat{\bt}_{h,p}^k}_{\Lambda_{h,p}^k} \le \frac{H}{2}\sqrt{2\log\left(\frac{H}{\alpha}\right) +  d_1\log\left(1 + \frac{KH^2}{\lambda_{\min,p}}\right)} + \sqrt{d_1 \lambda_{\max,p} }+ \nu_{p}
\end{align*}
and 
\begin{align*}
     \norm{\bt_{h,r}- \hat{\bt}_{h,r}^k}_{\Lambda_{h,r}^k} \le \frac{1}{2}\sqrt{2\log\left(\frac{H}{\alpha}\right) +  d_2\log\left(1 + \frac{K}{d_2\lambda_{\min,r}}\right) } + \sqrt{d_2 \lambda_{\max,r} }+ \nu_{r}.
\end{align*}

\end{proof}

\section{Useful Lemmas}

The following lemmas are the key components to prove Theorem~\ref{thm:priv_regret}.

\begin{lemma}[Slepian’s inequality; Lemma 5.33 in~\cite{vershynin2010introduction}]
\label{lem:Slepian}
    Consider two Gaussian processes $(X_t)_{t\in T}$ and $(Y_t)_{t\in T}$ whose increments satisfy the inequality $\ex{|X_s -X_t|^2} \le \ex{|Y_s - Y_t|^2}$ for all $s, t\in T$, then, $\ex{\sup_{t\in T} X_t}  \le \ex{\sup_{t\in T}Y_t}$.
\end{lemma}

\begin{lemma}[Concentration in the Gauss space]
\label{lem:cgauss}
    Let $f$ be a real valued $1$-Lipschitz function on $\mathbb{R}^n$. Let $X$ be a Gaussian random vector in $\mathbb{R}^n$ whose entries $X_i$ is $i.i.d$ $\mathcal{N}(0,\sigma^2)$. Then, for every $t\ge 0$,
    \begin{align*}
        \mathbb{P}\left\{f(X) - \ex{f(X)} \ge t \right\} \le \exp\left(-\frac{t^2}{2\sigma^2}\right)
    \end{align*}
\end{lemma}
\begin{proof}
    This result directly follows from Proposition 2.18 in~\cite{ledoux2001concentration} by considering $\gamma$ as the Gaussian measure on $\mathbb{R}^n$ with density of a multivariate normal distribution, i.e., $\mathcal{N}{(0,\sigma^2 I)}$ and hence $c$ in Proposition 2.18 of~\cite{ledoux2001concentration} is equal to $\frac{1}{\sigma^2}$. 
\end{proof}

\begin{lemma}[Operator norm of Gaussian random matrices]
\label{lem:norm_gaussian}
    Let $A$ be an $m \times n$ random matrix whose entries $A_{ij}$ are $i.i.d$ Gaussian random variables $\mathcal{N}(0,\sigma^2)$. Then, for any $t>0$, we have 
    \begin{align*}
        \mathbb{P}\{\norm{A} \ge \sigma(\sqrt{n} +\sqrt{m}) +t\} \le \exp\left(-\frac{t^2}{2\sigma^2}\right).
    \end{align*}
\end{lemma}
\begin{proof}
    We will first show that 
    \begin{align*}
        \ex{\norm{A}} \le \sigma(\sqrt{m} + \sqrt{n})
    \end{align*}
    by using Lemma~\ref{lem:Slepian}. To this end, note that the operator norm of $A$ can be computed as follows.
    \begin{align*}
        \norm{A} = \max_{u \in S^{n-1}, v\in S^{m-1}}\inner{Au}{v}.
    \end{align*}
    where $S^{n-1}$ and $s^{m-1}$ are unit spheres. Thus, $\norm{A}$ can be regarded as the supremum of the Gaussian process $X_{u,v}:=\inner{Au}{v}$ indexed by the pair of vectors $(u,v)\in S^{n-1} \times S^{m-1}$. Let us define $Y_{u,v}:=\inner{g}{u} + \inner{h}{v}$ where $g\in \mathbb{R}^n$ and $h \in \mathbb{R}^m$ are independent Gaussian random vectors whose entries are $i.i.d$ Gaussian random variables $\mathcal{N}(0,\sigma^2)$. Now, we compare the increments of these two Gaussian processes for every $(u,v),(u',v') \in S^{n-1} \times S^{m-1}$.
    \begin{align*}
        &\ex{|X_{u,v} -X_{u',v'}|^2}\\ &= \sigma^2\sum_{i=1}^n\sum_{j=1}^m |u_iv_j -u'_iv'_j|^2\\
        &=\sigma^2\norm{uv^T - u'v'^T}_{F}^2\\
        &=\sigma^2\norm{(u-u')v^T + u'(v-v')^T}_{F}^2\\
        &=\sigma^2\left(\norm{u-u'}_2^2 + \norm{v-v'}_2^2\right)\\
        & \quad+  2\sigma^2\text{trace}\left(v(u-u')^Tu'(v-v')^T\right)\\
        &=\sigma^2\left(\norm{u-u'}_2^2 + \norm{v-v'}_2^2 + 2(u^Tu' -1)(1-v^Tv') \right)\\
        &\le \sigma^2\left(\norm{u-u'}_2^2 + \norm{v-v'}_2^2\right)\\
        &= \ex{|Y_{u,v} -Y_{u',v'}|^2}.
    \end{align*}
    Therefore, Lemma~\ref{lem:Slepian} applies here, and hence
    \begin{align*}
        \ex{\norm{A}}& = \ex{\max_{(u,v)}{X_{u,v}}} \le \ex{\max_{(u,v)}{Y_{u,v}}}\\
        &= \ex{\norm{g}_2} + \ex{\norm{h}_2} \le \sigma(\sqrt{m} + \sqrt{n}).
    \end{align*}
    Finally, note that $\norm{A}$ is a $1$-Lipschitz function of $A$ when considered as a vector in $\mathbb{R}^{mn}$. Then, consider the function $f$ in Lemma~\ref{lem:cgauss} as the operator norm, we directly obtain the result.
 \end{proof}   
 
\begin{lemma}[Operator norm of symmetric Gaussian random matrices]
\label{lem:norm_sym_gaussian}
    Let $A$ be an $n \times n$ random symmetric matrix whose entries $A_{ij}$ on and above the diagonal are $i.i.d$ Gaussian random variables $\mathcal{N}(0,\sigma^2)$. Then, for any $t>0$, we have \begin{align*}
        \norm{A} \le 4\sigma\sqrt{n} + 2t,
    \end{align*}
    with probability at least $1-2\exp(-\frac{t^2}{2\sigma^2})$.
\end{lemma} 
 \begin{proof}
     Decompose the matrix $A$ into the upper triangular matrix $A^+$ and the lower triangular matrix $A^-$ such that $A = A^+ + A^-$. Note that without loss of generality, the entries on the diagonal are in $A^+$. Then, apply Lemma~\ref{lem:norm_gaussian} to both $A^+$ and $A^-$, and by a union bound, we have that the following inequalities hold simultaneously 
     \begin{align*}
         \norm{A^+} \le 2\sigma\sqrt{n} + t \quad\text{and}\quad \norm{A^-} \le 2\sigma\sqrt{n} + t
     \end{align*}
    with probability at least $1-2\exp(-\frac{t^2}{2\sigma^2})$ for any $t > 0$. Finally, by triangle inequality $\norm{A} \le \norm{A^+} + \norm{A^-}$, we prove the result.
 \end{proof}
 \begin{lemma}[Concentration of chi-square; Corollary to Lemma 1 of~\cite{laurent2000adaptive}]
 \label{lem:chi-square}
    Let $U$ be a $\chi^2$ statistic with $D$ degrees of freedom. Then, for any positive $x$,
    \begin{align*}
        \mathbb{P}\left\{ U \ge D + 2\sqrt{Dx} + 2x\right\} \le \exp(-x).
    \end{align*}
 \end{lemma}

\end{appendix}

\end{document}